%% file: main.tex
\documentclass[twoside,11pt]{article}
\usepackage{jair,rawfonts}
\usepackage{theapa}

\usepackage{amsmath}
\usepackage{xcolor}
\usepackage{bbm}
\usepackage{soul}
\usepackage{float}
\usepackage{algorithm}
\usepackage[linesnumbered,algoruled,algo2e,boxed,lined]{algorithm2e}
\SetKwInOut{Input}{Input}
\SetKwInOut{Output}{Output}
\usepackage{color}
\usepackage{amssymb}
\usepackage{tikz}

\usepackage{lineno}

\ShortHeadings{A Theoretical Perspective on Hyperdimensional Computing}{Thomas, Dasgupta, \& Rosing}
\firstpageno{1}

\newcommand{\bmm}[1]{\ensuremath{\mathbf{#1}}}

\renewcommand{\S}{\ensuremath{\mathcal{S}}}

\newcommand{\BlackBox}{\rule{1.5ex}{1.5ex}}  
\newenvironment{proof}{\par\noindent{\bf Proof\ }}{\hfill\BlackBox\\[2mm]}
 
\newtheorem{theorem}{Theorem}
\newtheorem{lemma}[theorem]{Lemma}

\newtheorem{corollary}[theorem]{Corollary}
\newtheorem{definition}[theorem]{Definition}

\arxivheading{}

\begin{document}

\title{A Theoretical Perspective on Hyperdimensional Computing}

\author{\name Anthony Thomas \email ahthomas@eng.ucsd.edu \\
\name Sanjoy Dasgupta \email dasgupta@eng.ucsd.edu \\
\name Tajana Rosing \email tajana@eng.ucsd.edu \\
\addr Department of Computer Science\\
University of California, San Diego\\
San Diego, CA 92093, USA
}


\maketitle

\begin{abstract}
\noindent Hyperdimensional (HD) computing is a set of neurally inspired methods for obtaining high-dimensional, low-precision, distributed representations of data. These representations can be combined with simple, neurally plausible algorithms to effect a variety of information processing tasks. HD computing has recently garnered significant interest from the computer hardware community as an energy-efficient, low-latency, and noise-robust tool for solving learning problems. In this review, we present a unified treatment of the theoretical foundations of HD computing with a focus on the suitability of representations for learning.
\end{abstract}


\section{Introduction}

Hyperdimensional (HD) computing is an emerging area at the intersection of computer architecture and theoretical neuroscience \shortcite{kanerva2009hyperdimensional}. It is based on the observation that brains are able to perform complex tasks using circuitry that: (1) uses low power, (2) requires low precision, and (3) is highly robust to data corruption. HD computing aims to carry over similar design principles to a new generation of digital devices that are highly energy-efficient, fault tolerant, and well-suited to natural information processing \shortcite{rahimi2018efficient}.

The wealth of recent work on neural networks also draws its inspiration from the brain, but modern instantiations of these methods have diverged from the desiderata above. The success of these networks has rested upon choices that are not neurally plausible, most notably significant depth and training via backpropagation. Moreover, from a practical perspective, training these models often requires high precision and substantial amounts of energy. While a large body of literature has sought to ameliorate these issues with neural networks, these efforts have largely been designed to address specific performance limitations. By contrast, the properties above emerge naturally from the basic architecture of HD computing.

Hyperdimensional computing focuses on the very simplest neural architectures. Typically, there is a single, static, mapping from inputs $x$ to much higher-dimensional ``neural'' representations $\phi(x)$ living in some space $\mathcal{H}$. All computational tasks are performed in $\mathcal{H}$-space, using simple, operations like element-wise additions and dot products. The mapping $\phi$ is often taken to be random, and the embeddings have coordinates that have low precision; for instance, they might take values $-1$ and $+1$. The entire setup is elementary and lends itself to fast, low-power hardware realizations.

Indeed, a cottage industry has emerged around developing optimized implementations of HD computing based algorithms on hardware accelerators \shortcite{imani2017exploring,rahimi2018efficient,gupta2018felix,manuel2019hardware,salamat2019f5,imani2019fach}. Broadly speaking, this line of work touts HD computing as an energy efficient, low-latency, and noise-resilient alternative to conventional realizations of general purpose ML algorithms like support vector machines, multilayer perceptrons, and nearest-neighbor classifiers. While this work has reported impressive performance benefits, there has been relatively little formal treatment of HD computing as a tool for general purpose learning.

This review has two broad aims. The first, more modest, goal is to introduce the area of hyperdimensional computing to a machine learning audience. The second is to develop a particular mathematical framework for understanding and analyzing these models. The recent literature has suggested a variety of different HD architectures that conform to the overall blueprint given above, but differ in many important details. We present a unified treatment of many such architectures that enables their properties to be compared. The most basic types of questions we wish to answer are:
\begin{enumerate}
\itemsep0em
\item How can individual items, sets of items, and sequences of items, be represented and stored in $\mathcal{H}$-space, in a manner that permits reliable decoding?
\item What kinds of noise can be tolerated in $\mathcal{H}$-space?
\item What kinds of structure in the input $x$-space are preserved by the mapping to $\mathcal{H}$-space?
\item What is the power of linear separators on the $\phi$-representation?
\end{enumerate}
Some of these questions have been introduced in the HD computing literature and studied in isolation \shortcite{plate2003holographic,gallant2013representing,kleyko2018classification,frady2018theory}. In this work we address these questions formally and in greater generality.

\section{Introduction to HD Computing}
\label{section:background}

In the following section we provide an introduction to the fundamentals of HD computing and provide some brief discussion of its antecedents in the neuroscience literature.

\subsection{High-Dimensional Representations in Neuroscience}
\label{subsec:expand-and-sparsify}

Neuroscience has proven to be a rich source of inspiration for the machine learning community: from the perceptron~\shortcite{rosenblatt.58}, which introduced a simple and general-purpose learning algorithm for linear classifiers, to neural networks~\shortcite{rumelhart.mcclelland.86}, to convolutional architectures inspired by visual cortex~\shortcite{fukushima1980neocognitron}, to sparse coding~\shortcite{olshausen.field.96} and independent component analysis~\shortcite{bell.sejnowski.95}. One of the most consequential discoveries from the neuroscience community, underlying much research at the intersection of neuroscience and machine learning, has been the notion of {\it high-dimensional distributed representations} as the fundamental data structure for diverse types of information. In the neuroscience context, these representations are also typically {\it sparse}. 

To give a concrete example, the sensory systems of many organisms have a critical component consisting of a transformation from relatively low dimensional sensory inputs to much higher-dimensional \emph{sparse} representations. These latter representations are then used for subsequent tasks such as recall and learning. In the olfactory system of the fruit fly~\shortcite{masse2009olfactory,turner2008olfactory,wilson2013early,caron2013random}, the mapping consists of two steps that can be roughly captured as follows:
\begin{enumerate}
    \item An input $\bmm{x} \in \mathbb{R}^{n}$ is collected via a sensory organ and mapped under a \emph{random linear transformation} to a point $\phi(\bmm{x}) \in \mathbb{R}^{d}$ ($d \gg n$) in a high-dimensional space.
    \item The coordinates of $\phi(\bmm{x})$ are ``sparsified'' by a thresholding operation which just retains the locations of the largest $k$ coordinates.
\end{enumerate}
In the fly, the olfactory input is a roughly 50-dimensional vector ($n=50$) corresponding to different types of odor receptor neurons while the sparse representation to which it is mapped is roughly 2,000-dimensional ($d=2000$). A similar ``expand-and-sparsify'' template is also found in other species, suggesting that this process somehow exposes the information present in the input signal in a way that is amenable to learning by the brain \shortcite{stettler2009representations,olshausen2004sparse,chacron2011efficient}. The precise mechanisms by which this occurs are still not fully understood, but may have close connections to some of the literature on the theory of neural networks and kernel methods~\shortcite{cybenko.89,barron.93,rahimi2008random}.

\subsection{HD Computing}
\label{subsec:hd-computing}

\begin{figure}
    \usetikzlibrary{positioning}
    \centering
    \begin{tikzpicture}[
    nocolor/.style={rectangle, draw, thick, minimum size=7mm},
    sharednode/.style={rectangle, draw=green!60, fill=green!10, thick, minimum size=6mm},
    hdnode/.style={rectangle, draw=orange!60, fill=orange!10, thick, minimum size=6mm},
    ]
    \node[align=center] (input) {Input Data: \\ $x \in \mathcal{X}$};
    \node[align=center,sharednode,dotted]  (encoding) [right=of input] {HD Encoding: \\ $\phi : \mathcal{X} \rightarrow \mathcal{H}$};
    \node[align=center,hdnode,dotted,minimum width=40mm] (memory) [right=of encoding] {Memory and \\ Data Structures $\in \mathcal{H}$};
    \node[align=center,sharednode,dotted,minimum height=12mm] (decoding) [right=of memory] {HD Decoding};
    \node[align=center,hdnode,dotted,anchor=mid,text centered,minimum width=40mm] (algorithms) [below= of memory] {HD Algorithms: \\ Learning/Reasoning};
    \node[align=center,minimum height=10mm,minimum width=18mm] (output) [below=of decoding] {Output};
    \node[align=center,minimum height=10mm,minimum width=18mm] (noise) [above=of memory] {Noise/Corruption: \\$\Delta \in \mathcal{H}$};
    
    \node[hdnode] (labelhd) [below=of input] {};
    \node[sharednode] (labelshared) [below=1mm of labelhd] {};
    \node[align=center] [right=1mm of labelhd] {Entirely in $\mathcal{H}$-space};
    \node[align=center] [right=1mm of labelshared] {Mixed in $\mathcal{H},\mathcal{X}$-space};
    
    \draw[->,line width=0.3mm] (input)      -- (encoding);
    \draw[->,line width=0.3mm] (encoding)   -- (memory);
    \draw[->,line width=0.3mm] (memory)     -- (decoding);
    \draw[->,line width=0.3mm] (decoding)   -- (output);
    \draw[->,line width=0.3mm] ([xshift=14mm]memory.south west) -- ([xshift=14mm]algorithms.north west);
    \draw[->,line width=0.3mm] ([xshift=-14mm]algorithms.north east) -- ([xshift=-14mm]memory.south east);
    \draw[->,line width=0.3mm] (algorithms.east) -- ([yshift=0.5mm]output.west);
    \draw[->,line width=0.3mm] (noise) -- (memory);
    \end{tikzpicture}
    \caption{The flow of data in HD computing. Data is mapped from the input space to HD-space under an encoding function $\phi : \mathcal{X} \rightarrow \mathcal{H}$. HD representations of data are stored in data structures and may be corrupted by noise or hardware failures. HD representations can be used as input for learning algorithms or other information processing tasks and may be decoded to recover the input data.}
    \label{fig:hd-overview}
\end{figure}
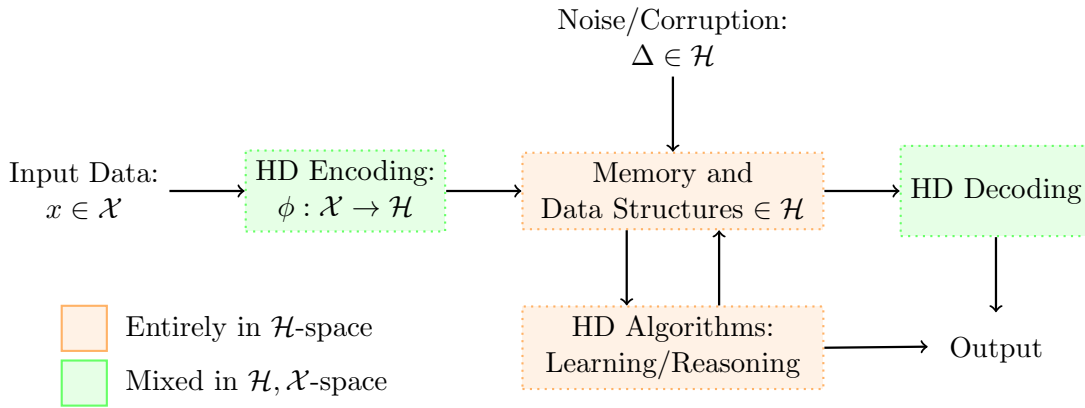

The notion of high-dimensional, distributed, data representations has engendered a number of computational models that have collectively come to be known as \emph{vector symbolic architectures} (VSA) \shortcite{levy2008vector}. In general, VSAs provide a systematic way to generate and manipulate high-dimensional representations of symbols so as to implement cognitive operations like association between related concepts. Notable examples of VSAs include ``holographic reduced representations'' \shortcite{plate1995holographic,plate2003holographic}, ``binary spatter codes'' \shortcite{kanerva1994spatter,kanerva1995family}, and ``matrix binding of additive terms'' \shortcite{gallant2013representing}. HD computing can be seen as a successor to these early VSA models, with a strong additional slant towards hardware efficiency. While our treatment focuses primarily on recent work on HD computing, many of our results apply to these earlier VSA models as well.

An overview of data-flow in HD computing is given in Figure \ref{fig:hd-overview}. The first step in HD computing is encoding, which maps a piece of input data to its high-dimensional representation under some function $\phi : \mathcal{X} \rightarrow \mathcal{H}$. The nature of $\phi$ depends on the type of input and the choice of $\mathcal{H}$. In this review, we consider inputs consisting of sets, sequences, and structures composed from a finite alphabet as well as vectors in a Euclidean space. The space $\mathcal{H}$ is some $d$-dimensional inner-product space defined over the real numbers or a subset thereof. Work in the literature on both HD computing and traditional neural networks has also explored complex-valued embeddings \shortcite{weiss2016neural,parcollet2018quaternion,zhang2021beyond}. However, we here focus on the more common case of real-valued embeddings. For computational reasons, it is common to restrict $\mathcal{H}$ to be defined over integers in a limited range $[-b,b]$. We emphasize that the dimension of $\mathcal{H}$ need not, in general, be greater than that of $\mathcal{X}$. Indeed, in several cases the encoding methods discussed can be used to reduce the dimension of the data.

The HD representations of data can be manipulated using simple element-wise operators. Two common and important such operations are ``bundling'' and ``binding.'' The bundling operator is used to compile a set of elements in $\mathcal{H}$ and takes the form of a function $\oplus : \mathcal{H} \times \mathcal{H} \rightarrow \mathcal{H}$. The function takes two points in $\mathcal{H}$ and returns a third point that is similar to both operands. The binding operator is used to create ordered tuples of points in $\mathcal{H}$ and is likewise a function $\otimes : \mathcal{H} \times \mathcal{H} \rightarrow \mathcal{H}$. The function takes a pair of points in $\mathcal{H}$ as input, and produces a third point dissimilar to both operands. We make these notions more precise in our subsequent discussion of encoding.

Given the HD representation $\phi(\mathcal{S})$ of a set of items $\mathcal{S} \subset \mathcal{X}$ (produced by bundling the items), we may be interested to query the representation to determine if it contains the encoding of some $x \in \mathcal{X}$. To do so, we compute a metric of similarity $\rho(\phi(x), \phi(\mathcal{S}))$ and declare that the item is present in $\mathcal{S}$ if the similarity is greater than some critical value. This process can be used to decode the HD representation so as to recover the original points in $\mathcal{X}$ \shortcite{plate2003holographic,frady2018theory}. We may additionally wish to assert that we can decode reliably even if $\phi(\mathcal{S})$ has been corrupted by some noise process. One of our chief aims in this paper is to mathematically characterize sufficient conditions for robust decoding under different noise models and input data types.

Beyond simply storing and recalling specific patterns, HD representations may also be used for learning. HD computing is most naturally applicable to classification problems. Suppose we are given some collection of labeled examples $\mathcal{S} = \{(x_{i}, y_{i})\}_{i=1}^{N}$, where $x_{i} \in \mathcal{X}$ and $y_{i} \in \{c_{i}\}_{i=1}^{K}$ is a categorical variable indicating the class label of a particular $x_{i}$. One simple form of HD classification bundles together the data corresponding to a particular class to generate a ``prototypical'' example for the class \shortcite{kanerva2009hyperdimensional,kleyko2018classification,rahimi2018efficient}:
\begin{gather}
    \label{eqn:training-bundle}
    \phi(c_{k}) = \bigoplus_{i \,:\, y_{i} = c_{k}} \phi(x_{i})
\end{gather}
The resulting $\phi(c_{k})$ are sometimes quantized to lower precision or sparsified via a thresholding operation. A nice feature of this scheme is that it is extremely simple to implement in an on-line fashion: that is, on streaming data arriving continuously over time \shortcite{rahimi2018efficient}. It is common to fine-tune the class prototypes using a few rounds of perceptron training \shortcite{imani2017voicehd,imani2019adapthd}. Given some subsequent piece of query data $x_{q} \in \mathcal{X}$ for which we do not know the correct label, we simply return the label of the most similar prototype:
\[
    k^{\star} = \underset{k \in 1,\ldots,K}{\text{argmax }} \rho(\phi(x_{q}),\phi(c_{k})).
\]
The similarity metric $\rho$ is typically taken to be the dot-product, with the operands normalized if necessary. Thus, on the whole, the scheme is quite similar to classical statistical methods like naive Bayes and Fisher's linear discriminant. In Section \ref{subsec:hd-kernels}, we consider properties of the HD encoding that can make linear models more powerful in HD space than in the original space.

HD computing and closely related techniques have been applied to a wide variety of practical problems in fields ranging from bio-signal processing \shortcite{rahimi2016biosignal,imani2017voicehd}, to natural language processing \shortcite{sahlgren2005introduction}, and robotics \shortcite{mitrokhin2019learning,neubert2019introduction}. We are here concerned with a more abstract treatment that focuses on the basic properties of HD computing and will not attempt to survey this literature. The interested reader is referred to \shortcite{rahimi2016biosignal,kleyko2018classification} for discussions related to practical aspects of HD computing.

\section{Encoding and Decoding Discrete Data}
\label{section:set-encoding}

A central object in HD computing is the mapping from inputs to their high-dimensional representations. The design of this mapping, typically referred to as ``encoding'' in the literature on HD computing, has been the subject of considerable research. There is a wide range of possible encoding methods. Some of these have been introduced in the HD computing literature and studied in isolation \shortcite{plate2003holographic,gallant2013representing,kleyko2018classification}. In this review, we present a novel unifying framework in which to study these mappings and to characterize their key properties in a non-asymptotic setting. We first discuss the encoding and decoding of \emph{sets} in some detail. Many HD encoding procedures for more complex data types such as sequences essentially amount to transforming the data into a set and then applying the standard set-encoding method.

\subsection{Finite Sets}

\noindent Let $\mathcal{A} = \{a_{i}\}_{i=1}^{m}$ be some finite alphabet of $m$ symbols. Symbols $a \in \mathcal{A}$ are mapped to $\mathcal{H}$ under an encoding function $\phi : \mathcal{A} \rightarrow \mathcal{H}$. Our goal in this section is to consider the encoding of sets $\mathcal{S}$ whose elements are drawn from $\mathcal{A}$. The HD representation of $\mathcal{S}$ is constructed by superimposing the embeddings of the constituent elements using the bundling operator $\oplus : \mathcal{H} \times \mathcal{H} \rightarrow \mathcal{H}$. The encoding of $\mathcal{S}$ is defined to be $\phi(\mathcal{S}) = \oplus_{a\in\mathcal{S}} \phi(a)$. We first focus on the intuitive setting in which $\oplus$ is the element-wise sum and then address other forms of bundling.

To determine if some $a \in \mathcal{A}$ is contained in $\mathcal{S}$, we check if the dot product $\langle \phi(a), \phi(\mathcal{S}) \rangle$ exceeds some fixed threshold. If the codewords $\{\phi(a): a \in \mathcal{A}\}$ are orthogonal and have a constant length $L$, then we have $\langle \phi(a), \phi(\mathcal{S}) \rangle = L^2 \, \mathbbm{1}(a \in \mathcal{S})$, where $\mathbbm{1}$ is the indicator function which evaluates to one if its argument is true and zero otherwise. However, when the codewords are not perfectly orthogonal, we have $\langle \phi(a), \phi(\S) \rangle = L\mathbbm{1}(a \in \mathcal{S}) + \Delta$, where $\Delta$ is the ``cross-talk'' caused by interference between the codewords. In order to decode reliably, we must ensure the contribution of the cross-talk is small and bounded. We formalize this using the notion of incoherence popularized in the sparse coding literature. We define incoherence formally as \shortcite{donoho2005stable}:
\begin{definition}
\label{def:coherence}
\textbf{Incoherence}. For $\mu \geq 0$, we say $\phi: \mathcal{A} \to \mathcal{H}$ is $\mu$-incoherent if for all distinct $a,a' \in \mathcal{A}$, we have
\[
|\langle \phi(a), \phi(a') \rangle| \leq \mu L^2
\]
where $L = \min_{a \in \mathcal{A}} \|\phi(a)\|$.
\end{definition}
\noindent When $d \geq m$, it is possible to have codewords that are mutually orthogonal, whereupon $\mu = 0$. In general, we will be interested in results that do not require $d \geq m$.

\subsubsection{Exact Decoding of Sets}

In the following section, we show how the cross-talk can be bounded in terms of the incoherence of $\phi$, and use this to derive a simple threshold rule for exact decoding. 
\begin{theorem}
\label{thm:decodability}
Let $L = \min_{a \in \mathcal{A}} \|\phi(a)\|$ and let the bundling operator be the element wise sum. To decode whether an element $a$ lies in set $S$, we use the rule
\[
\langle \phi(a), \phi(S) \rangle \geq \frac{1}{2} L^2.
\]
This gives perfect decoding for sets of size $\leq s$ if $\phi$ is $1/(2s)$-incoherent.
\end{theorem}
\begin{proof}
Consider some symbol $a$. Then:
\[
    \langle \phi(a), \phi(\S) \rangle = \mathbbm{1}(a \in \mathcal{S})\langle \phi(a), \phi(a) \rangle + \sum_{a' \in \mathcal{S}\setminus\{a\}} \langle \phi(a), \phi(a') \rangle
\]
If $a \in \mathcal{S}$, then the above is lower bounded by $L^{2} - sL^{2}\mu$, where $\mu$ is the incoherence of $\phi$. Otherwise, it is upper bounded by $sL^{2}\mu$. So we decode perfectly if $sL^{2}\mu < L^2/2$, or $\mu < 1/(2s)$.
\end{proof}

\subsubsection{Random Codebooks}

\noindent In practice, each $\phi(a)$ is usually generated by sampling from some distribution over $\mathcal{H}$ or a subset of $\mathcal{H}$ \shortcite{kanerva2009hyperdimensional,kleyko2018classification,rahimi2018efficient}. We typically require that this distribution is factorized so that coordinates of $\phi(a)$ are i.i.d.. Intuitively, the incoherence condition stipulated in Theorem \ref{thm:decodability} will hold if dot products between two different codewords are concentrated around zero. Furthermore, we would like it to be the case that this concentration occurs quickly as the encoding dimension is increased. It turns out that a fairly broad family of simple distributions satisfies these properties.

As an example, suppose $\phi(a)$ is sampled from the uniform distribution over $\{\pm 1\}^{d}$, which we denote $\phi(a)\sim\{\pm 1\}^{d}$. In this case, $L = \sqrt{d}$ exactly, and a direct application of Hoeffding's inequality and the union bound yields:
\[
    \mathbb{P}(\text{$\exists$ distinct $a,a' \in \mathcal{A}$ s.t.\ } |\langle \phi(a),\phi(a') \rangle| \geq \mu d) \leq m^{2}\exp\left( - \frac{\mu^{2}d}{2} \right).
\]
(Recall that $m = |\mathcal{A}|$.) Stated another way, with high probability $\mu = O(\sqrt{(\ln m)/d})$, meaning that we can make $\mu$ as small as desired by increasing $d$.

In fact, the same basic approach holds for the much broader class of \emph{sub-Gaussian} distributions, which can be characterized as follows \shortcite{wainwright2019high}:
\begin{definition}
\textbf{Sub-Gaussian Random Variable}. A random variable $X \sim P_{X}$ is said to be sub-Gaussian if there exists $\sigma \in \mathbb{R}^{+}$, referred to as the sub-Gaussian parameter, such that:
\[
    \mathbb{E}[\exp \left( \lambda(X - \mathbb{E}[X]) \right)] \leq \exp\left( \frac{\sigma^{2}\lambda^{2}}{2}\right) \text{ for all } \lambda \in \mathbb{R}.
\]
\end{definition}
Intuitively, the tails of a sub-Gaussian random variable decay at least as fast those of a Gaussian. We say the encoding $\phi$ is $\sigma$-sub-Gaussian if $\phi(a)$ is generated by sampling its $d$ coordinates independently from the same sub-Gaussian distribution with parameter $\sigma$. We say $\phi$ is ``centered'' if the distribution from which it is sampled is of mean zero. In general, we assume $\phi$ is centered unless stated otherwise.

Codewords drawn from a sub-Gaussian distribution have the useful property that their lengths concentrate fairly rapidly around their expected value. This concentration is, in general, worse than sub-Gaussian but well behaved nonetheless. The following result is well known but we reiterate it here as it is useful for our subsequent discussion. A proof is available in the appendix.
\begin{theorem}
    \label{thm:sub-gaussian-lengths}
    Let $\phi$ be centered and $\sigma$-sub-Gaussian. Then:
    \[
        \mathbb{P}(\exists\, a \in \mathcal{A} \text{ s.t. } |\|\phi(a)\|_{2}^{2} - \mathbb{E}[\|\phi(a)\|_{2}^{2}]| \geq t) \leq 2m\exp\left( -c\min\left\{\frac{t^{2}}{d\sigma^{4}}, \frac{t}{\sigma^{2}} \right\} \right)
    \]
    for some positive absolute constant $c$.
\end{theorem}

Like the conventional Gaussian distribution, sub-Gaussianity is preserved under linear transformations. That is, if $\bmm{x} = \{x_{i}\}_{i=1}^{n}$ is a sequence of i.i.d. sub-Gaussian random variables and $\bmm{a}$ is an arbitrary vector in $\mathbb{R}^{n}$, then $\langle \bmm{a}, \bmm{x} \rangle$ is sub-Gaussian with parameter $\sigma\|a\|_{2}$ \shortcite{wainwright2019high}. We can obtain a more general version of the previous result about $\phi \sim \{\pm 1\}^{d}$ which applies to $\phi(a)$ sampled from any sub-Gaussian distribution.
\begin{theorem}
\label{thm:incoherence}
Let $\phi$ be $\sigma$-sub-Gaussian. Then, for $\mu > 0$,
\[
    \mathbb{P}(\exists\, \text{{\upshape distinct }} a,a' \in \mathcal{A} \text{ s.t. } |\langle \phi(a), \phi(a') \rangle| \geq \mu L^{2}) \leq m^{2}\exp\left( -\frac{\mu^{2}\kappa L^{2}}{2\sigma^{2}} \right)
\]
where $\kappa = (\min_a \|\phi(a)\|^2)/(\max_a \|\phi(a)\|^2)$.
\end{theorem}

\begin{proof}
Fix some $a$ and $a'$. Treating $\phi(a)$ as a fixed vector in $\mathbb{R}^{d}$ and using the fact that sub-Gaussianity is preserved under linear transformations, we may apply a Chernoff bound for sub-Gaussian random variables (e.g. Prop 2.1 of \shortcite{wainwright2019high}) to obtain:
\[
    \mathbb{P}(|\langle \phi(a), \phi(a') \rangle| \geq \mu L^{2}) \leq 2\exp\left( -\frac{\mu^{2} L^{4}}{2\sigma^{2}\|\phi(a)\|_{2}^{2}} \right) \leq 2\exp\left( -\frac{\mu^{2}L^{4}}{2\sigma^{2}L_{\max}^{2}} \right)
\]
where $L_{\max} = \max_{a\in\mathcal{A}}\|\phi(a)\|_{2}$. Therefore, taking $\kappa = L^{2}/L_{\max}^{2}$, we have:
\[
\mathbb{P}(|\langle \phi(a), \phi(a') \rangle| \geq \mu L^{2}) \leq 2\exp\left( -\frac{\mu^{2}\kappa L^{2}}{2\sigma^{2}} \right)
\]
and the claim follows by applying the union bound over all $\binom{m}{2} < m^{2}/2$ pairs of codewords. We note that, per Theorem \ref{thm:sub-gaussian-lengths}, $\kappa \to 1$ as $d$ becomes large.
\end{proof}
To be concrete and provide useful practical guidance, we here introduce three running examples of codeword distributions. 

\vspace{2mm}
\noindent \textbf{Dense Binary Codewords}. In our first example, the most common in practice in our impression, $\phi(a)$ is sampled from the uniform distribution over the $d$-dimensional unit cube $\{-1,+1\}^d$. This approach is advantageous because it leads to efficient hardware implementations \shortcite{imani2017voicehd,rahimi2018efficient} and is simple to analyze. 

\vspace{2mm}
\noindent\textbf{Gaussian Codewords}. Our second example consists of codewords sampled from the $d$-dimensional Gaussian distribution \shortcite{plate2003holographic}. That is, $\phi(a) \sim \mathcal{N}(\bmm{0}_{d}, \sigma^{2}\bmm{I}_{d})$, where $\bmm{0}_{d}$ is the $d$-dimensional zero vector. Here, the codewords will not be of exactly the same length. However, Theorem \ref{thm:sub-gaussian-lengths} ensures that squared codeword lengths are concentrated around their expected value of $\sigma^{2}d$. More formally, for some $\tau > 0$:
\[
    \mathbb{P}(\exists\, a \in \mathcal{A} \text{ s.t. } |\|\phi(a)\|_{2}^{2} - \sigma^{2}d| \geq \tau\sigma^{2}d) \leq 2m\exp\left( -c\min\left\{\tau^{2}d, \tau d\right\} \right).
\]

In both cases, we can see that to obtain a $\mu$-incoherent codebook with probability $1-\delta$, is it sufficient to choose:
\[
    d = O\left(\frac{2}{\mu^{2}}\ln\frac{m}{\delta}\right)
\] 
Or, stated another way, we have $\mu = O(\sqrt{(\ln m) / d})$ with high probability. The key point in the two examples above is that the encoding dimension is inversely proportional to $\mu^{2}$. Per Theorem \ref{thm:decodability}, to decode correctly it is sufficient to have $\mu = 1/(2s)$, meaning that the encoding dimension scales quadratically with the number of elements in the set, but only logarithmically in the alphabet size and probability of error.

We will also consider a third example in which the codewords are \emph{sparse} and binary. However, we defer this for the time being as slightly different encoding methods and analysis techniques are appropriate.

\subsubsection{Decoding with Small Probability of Error}
\label{subsec:low-error-decoding}

The analysis above gives strong \emph{uniform} bounds showing that, with probability at least $1-\delta$ over random choice of the codebook, \emph{every} subset of size at most $s$ will be correctly decoded. However, this guarantee requires us to impose the unappealing restriction that $s \ll \sqrt{d}$ which is a significant practical limitation. We here show that we can obtain $s = O(d)$ but with a weaker \emph{pointwise} guarantee: any arbitrarily chosen set of size at most $s$ will be correctly decoded with probability $1-\delta$ over the random choice of codewords. Rather than insist on a hard upper bound on the incoherence of the codebook, we can instead require the milder condition that random sums over dot-products between $\leq s$ codewords are small with high-probability. We define this property more formally as follows:
\begin{definition}
\textbf{Subset Incoherence}. For $\tau > 0$, we say a random mapping $\phi : \mathcal{A} \rightarrow \mathcal{H}$ satisfies $(s,\tau,\delta)$-subset incoherence if, for any $\S \subset \mathcal{A}$ of size at most $s$, with probability at least $1-\delta$ over the choice of $\phi$:
\[
    \underset{a \notin \S}{\max}\, \left| \sum_{a' \in \S} \langle \phi(a), \phi(a') \rangle \right| \leq \tau L^{2}
\]
where $L = \min_{a \in \mathcal{A}} ||\phi(a)||$.
\end{definition}
Once again, it turns out that sampling the codewords from a sub-Gaussian distribution can readily be seen to satisfy a subset-incoherence condition with high-probability:
\begin{theorem}
Let $\phi$ be $\sigma$-sub-Gaussian and fix some $\mathcal{S} \subset \mathcal{A}$ of size $s$. Then 
\[
\mathbb{P}\left(\underset{a \notin \S}{\max}\, \left| \sum_{a' \in \S} \langle \phi(a), \phi(a') \rangle \right| \geq \tau L^{2} \right) \leq 2m\exp\left( -\frac{\kappa\tau^{2}L^{2}}{2s\sigma^{2}} \right)
\]
where $\kappa$ and $L$ are as in Theorem \ref{thm:incoherence}.
\end{theorem}
The proof is similar to Theorem \ref{thm:incoherence} and is available in the appendix. As a concrete example, in the practically relevant case that $\phi \sim \{\pm 1\}^{d}$ the above boils down to:
\[
    \mathbb{P}\left(\exists\,a \notin \S \text{ s.t. } \left| \sum_{a' \in \S} \langle \phi(a), \phi(a') \rangle \right| \geq \tau d \right) \leq 2m\exp\left( -\frac{\tau^{2}d}{2s} \right).
\]
Stated another way, we have: $\tau = O(\sqrt{(s\ln m)/d})$. Following Theorem \ref{thm:decodability}, in order to ensure correct decoding with high probability, we must simply argue that the codebook satisfies the subset-incoherence property with $\tau = 1/2$, meaning we should choose the encoding dimension to be $d = O(s \ln m)$.

This method of analysis is similar to that of \shortcite{plate2003holographic,gallant2013representing,frady2018theory}, who reach the same conclusion vis-\`{a}-vis linear scaling using the central limit theorem. However, our formalism is more general and is non-asymptotic.

\subsubsection{Comparing Set Representations}

We can estimate the size of a set by computing the norm of its encoding, where the precision of the estimate can be bounded in terms of the incoherence of $\phi$. In the following discussion, we make the simplifying assumption that the codewords are all of a constant length $L$. Again appealing to Theorem \ref{thm:sub-gaussian-lengths}, we can see that this assumption is not onerous since the codeword lengths concentrate around their expected value.

\begin{theorem}
\label{thm:set-size}
Let $\mathcal{S}$ be a set of size $s$. Then:
\[
    s(1 - s\mu) \leq \frac{1}{L^{2}}\|\phi(\mathcal{S})\|_{2}^{2} \leq s(1 + s\mu) 
\]
\end{theorem}
\begin{proof}
The proof is by direct manipulation:
\begin{align*}
    \frac{1}{L^{2}}\|\phi(\mathcal{S})\|_{2}^{2} &= \frac{1}{L^{2}}\langle \phi(\mathcal{S}), \phi(\mathcal{S}) \rangle = \frac{1}{L^{2}}\sum_{a\in\mathcal{S}}\langle \phi(a),\phi(a) \rangle + \frac{1}{L^{2}}\sum_{a,a'\neq a \in \mathcal{S}} \langle \phi(a), \phi(a') \rangle \\
    &\leq \frac{1}{L^{2}}(sL^{2} + s^{2}\mu L^{2}) .
\end{align*}
The other direction is analogous.
\end{proof}
Given a pair of sets $\mathcal{S},\mathcal{S}'$ over the same alphabet, we can estimate the size of their intersection and union directly from their encoded representation.
\begin{theorem}
\label{thm:intersection}
Let $\mathcal{S}$ and $\mathcal{S}'$ be sets of size $s$ and $s'$ drawn from $\mathcal{A}$ and denote their encodings by $\phi(\mathcal{S})$ and $\phi(\mathcal{S}')$ respectively.
\[
      |\mathcal{S} \cap \mathcal{S}'| - ss'\mu \leq \frac{1}{L^{2}}\langle \phi(\mathcal{S}), \phi(\mathcal{S}') \rangle \leq |\mathcal{S} \cap \mathcal{S}'| + ss'\mu
\]
\end{theorem}
\noindent The proof is similar to Theorem \ref{thm:set-size} and is deferred to the appendix. Noting as well that $|\mathcal{S} \cup \mathcal{S}'| = |\mathcal{S}| + |\mathcal{S}'| - |\mathcal{S} \cap \mathcal{S}'|$, we see that we can estimate the size of the union using the previous theorem. In practice, it may be unnecessary to compute these quantities with a high degree of precision. For instance, it may only be necessary to identify sets with a large intersection-over-union. Provided the definition of ``large'' is somewhat loose, we can accept a higher incoherence among the codewords in exchange for reducing the encoding dimension.

\subsubsection{Sparse and Low-Precision Encodings}
\label{section:sparse-codewords}

In the previous discussion, we assumed the bundling operator was the element-wise sum. This is a natural choice when the codewords are dense or non-binary. However, the resulting encodings are of unconstrained precision which may be undesirable from a computational perspective. For the purposes of representing sets of size $\leq s$, we may truncate $\phi(\mathcal{S})$ to lie in the range $[-c,c]$, with negligible loss in accuracy provided $c = O(\sqrt{s})$. In practice, it is common to quantize the encodings more aggressively to binary precision by thresholding \shortcite{kanerva1994spatter,rahimi2017high,burrello2018one,imani2019binary}. In other words, we encode as $\phi(\mathcal{S}) = g_{t}(\mathcal{S})$, where $g_{t}$ is a thresholding function that is applied coordinate-wise: $g_{t}(x) = 1$ if $x \geq t$ and $0$ otherwise.

As a notable special case of the thresholding rule described above, we here consider encoding with \emph{sparse} codewords. In this case, we assume that a coordinate in a codeword is non-zero with some small probability. In other words, $\phi(a)_{i} \sim \text{Bernoulli}(p)$, where $p \ll 1/2$. We may then bundle items by taking an element-wise sum of their codewords with threshold $t = 1$, which is equivalent to taking the element-wise maximum over the codewords. That is, $\phi(\mathcal{S}) = \max_{a \in \mathcal{S}} \phi(a)$, where the $\max$ operator is applied coordinate-wise. Noting that the max is upper bounded by the sum in this setting, the notion of incoherence is a relevant quantity and the analysis of Theorem \ref{thm:decodability} continues to apply.

This encoding procedure is essentially a standard implementation of the popular ``Bloom filter'' data structure for representing sets \shortcite{bloom1970space}. The conventional Bloom filter differs slightly in that the typical decoding rule is to threshold $\langle \phi(a), \phi(\S) \rangle$ at $\|\phi(a)\|_{1}$. There is a large literature on Bloom filters with applications ranging from networking and database systems to neural coding, and several schemes for generating good codewords have been proposed \shortcite{broder2004network,pagh2005optimal,dasgupta2018neural}. Using the random coding scheme described here, the optimal value of $p$ can be seen to be $(\ln 2)/s$ and, to ensure the probability of a false positive is at most $\delta$, the encoding dimension should be chosen on the order of $s\ln(1/\delta)$ \shortcite{broder2004network}. A practical benefit of Bloom filters is that they have an efficient implementation using hash functions which does not require materializing a codebook as in methods based on random sampling. This may be beneficial when the alphabet size is large enough that storing codewords is not possible. The connections between HD computing and Bloom filters are examined in greater detail in \shortcite{kleyko2019autoscaling}.

We remark that this method of encoding is related to an interesting procedure known as ``context dependent thinning'' (CDT) which can be used to control the density of binary representations  \shortcite{rachkovskij2001representation,kleyko2018classification}. CDT takes the logical ``and'' of $\phi(\mathcal{S})$ and some permutation $\sigma(\phi(\mathcal{S}))$ to obtain the thinned representation $\phi(\mathcal{S})' = \phi(\mathcal{S})\land\sigma(\phi(\mathcal{S}))$. This process can be repeated until the desired density of $\phi(\mathcal{S})$ is achieved. A capacity analysis of CDT representations can be found in \shortcite{kleyko2018classification}.

\subsection{Robustness to Noise}
\label{section:discrete-robustness}

In this section we explore the noise robustness properties of the encoding methods discussed above using the formalism of incoherence. We consider some unspecified noise process which corrupts the encoding of a set $\mathcal{S} \subset \mathcal{A}$ of size at most $s$ according to $\tilde{\phi}(\mathcal{S}) = \phi(\mathcal{S}) + \Delta_{\S}$. We say $\Delta_{\mathcal{S}}$ is $\rho$-bounded if:
\[
        \underset{a \in \mathcal{A}}{\max}\,|\langle \phi(a), \Delta_{\S} \rangle| \leq \rho.
\]
We are interested in understanding the conditions under which we can still decode reliably.
\begin{theorem}
\label{thm:noisy-decoding}
Suppose $\mathcal{S}$ has size $\leq s$ and $\Delta_{\mathcal{S}}$ is $\rho$-bounded. We can correctly decode $\mathcal{S}$ using the thresholding rule from Theorem \ref{thm:decodability} if:
\[
    \frac{\rho}{L^{2}} + s\mu < \frac{1}{2}
\]
where $L = \min_{a\in\mathcal{A}}\|\phi(a)\|_{2}$.
\end{theorem}
The proof is a straightforward extension of Theorem \ref{thm:decodability} and is available in the appendix. The practical implication is that there is a tradeoff between the incoherence of the codebook and robustness to noise: a higher incoherence allows for a smaller encoding dimension but at the cost of a tighter constraint on $\rho$. We can analyze several practically relevant noise models by placing additional restrictions on $\Delta_{\S}$ and by considering worst or typical case bounds on $\rho$. We here consider different forms of noise under constraints on $\mathcal{H}$. Our goal is to understand how the magnitude of noise that can be tolerated scales with the encoding dimension, size $s$ of the encoded set, and size $m$ of the alphabet. In each setting we consider a ``passive'' model in which the noise is sampled randomly from some distribution, and an ``adversarial'' model in which the noise is arbitrary and may be designed to maliciously corrupt the encodings. We again appeal to Theorem \ref{thm:sub-gaussian-lengths} to justify a simplifying assumption that the codewords are of equal length.

\begin{lemma}
\textbf{Sub-Gaussian Codewords}. Fix a centered and $\sigma$-sub-Gaussian codebook $\phi$ whose codewords are of length $L$. Consider the passive additive white Gaussian noise model $\Delta_{\S} \sim \mathcal{N}(0,\sigma_{\Delta}^{2}\bmm{I}_{d})$; that is, each coordinate is corrupted by Gaussian noise with mean zero and variance $\sigma_\Delta^2$. Then, we can correctly decode with probability $1-\delta$ over random draws of $\Delta_{\S}$ provided:
\[
    \sigma_{\Delta} < \frac{L}{\sqrt{2\ln(2m/\delta)}}\left(\frac{1}{2} - s\mu\right)
\]
Now consider an adversarial model in which $\Delta_{\S}$ is arbitrary save for a constraint on the norm: $\|\Delta_{\S}\|_{2} \leq \omega L$. Then, we can correctly decode provided:
\[
    \omega < \frac{1}{2} - s\mu
\]
\end{lemma}
\begin{proof}
Let us first consider the passive case in which $\Delta_{\S} \sim \mathcal{N}(0,\sigma_{\Delta}^{2}\bmm{I}_{d})$. Fix some $a \in \mathcal{A}$. Then $\langle \phi(a),\Delta_{\S} \rangle \sim \mathcal{N}(0,\sigma_{\Delta}^{2}L^{2})$. By a standard tail bound on the Gaussian distribution \shortcite{wainwright2019high} and the union bound, we have:
\[
    \mathbb{P}(\exists\, a \text{ s.t. } |\langle \phi(a), \Delta_{\S} \rangle| \geq \rho) \leq 2m\exp\left(-\frac{\rho^{2}}{2\sigma_{\Delta}^{2}L^{2}}\right).
\]
Therefore, with probability $1-\delta$, we have that $\Delta_{\S}$ is $\rho$-bounded for
\[
    \rho \leq \sigma_{\Delta}L\sqrt{2\ln(2m/\delta)}.
\]
By Theorem \ref{thm:noisy-decoding} we can decode correctly if:
\begin{gather*}
    \frac{\sigma_{\Delta}L\sqrt{2\ln(2m/\delta)}}{L^{2}} + s\mu < \frac{1}{2} \Rightarrow \sigma_{\Delta} < \frac{L}{\sqrt{2\ln (2m/\delta)}}\left(\frac{1}{2} - s\mu\right).
\end{gather*}
Now consider the adversarial case in which $\|\Delta_{\S}\|_{2} \leq \omega L$. By the Cauchy-Schwarz inequality, $|\langle \phi(a),\Delta_{\S} \rangle| \leq \omega L^{2}$. Therefore, by Theorem \ref{thm:noisy-decoding}, we can decode correctly if
\begin{gather*}
    \frac{\omega L^{2}}{L^{2}} + s\mu < \frac{1}{2} \Rightarrow \omega < \frac{1}{2} - s\mu.
\end{gather*}
\end{proof}
We again emphasize that, per Theorem \ref{thm:incoherence}, $\mu = O(\sqrt{(\ln m) / d})$. Since $L = O(\sqrt{d})$, we can see that we can tolerate $\sigma_{\Delta} \approx \sqrt{d/(\ln m)} - s$ in the passive case. We next turn to a notable special case of the above in which the codewords are dense and binary. In this case, we may assume that $\mathcal{H}$ is constrained to be integers in the range $[-s,s]$.
\begin{lemma}
\label{lemma:dense-binary}
\textbf{Dense Binary Codewords}. Fix a codebook $\phi$ such that $\phi(a) \sim \{\pm 1\}^{d}$ for each $a \in \mathcal{A}$. Consider a passive noise model in which $\Delta_{\S} \sim \text{\rm unif}(\{-c,...,c\}^{d})$; that is, each coordinate is shifted by an integer amount chosen uniformly at random between $-c$ and $c$. Then, we can correctly decode with probability $1-\delta$ provided:
\[
    c < \sqrt{\frac{d}{2\ln (2m/\delta)}}\left(\frac{1}{2} - s\mu\right)
\]
Now consider an adversarial model in which we assume $\|\Delta_{\S}\|_{1} \leq \omega s d$. Then we can decode correctly if:
\[
    \omega < \frac{1}{2s} - \mu.
\]
\end{lemma}

\noindent A proof is available in the Appendix. We next consider the case of Section \ref{section:sparse-codewords} in which the codewords are sparse and binary and the bundling operator is the element-wise maximum. We here assume that $\tilde{\phi}(\S) = \phi(\S) + \Delta_{\S}$ is truncated so that each coordinate is either $0$ or $+1$.
\begin{lemma}
\textbf{Sparse Binary Codewords}. Fix a codebook $\phi$ such that $\phi(a) \in \{0,1\}^{d}$, and assume some fraction $p \ll 1/2$ of coordinates are non-zero for each $a \in \mathcal{A}$. Consider a passive noise model in which:
\[
    \Delta_{\S} \sim \begin{cases} -1 &\text{ w.p. } \frac{\theta}{2} \\ 0 &\text{ w.p. } 1 - \theta \\ +1 &\text{ w.p. } \frac{\theta}{2}. \end{cases}
\]
Then we can decode correctly with probability $1-\delta$ provided:
\[
    \theta < \frac{1}{2} - 2s\mu - \sqrt{\frac{1}{2dp}\ln\frac{2m}{\delta}}.
\]
Now consider an adversarial model in which $\|\Delta_{\S}\|_{1} \leq \omega d$. Then we can decode correctly if $\omega < p(\frac{1}{2} - s\mu)$.
\end{lemma}
\begin{proof}
Consider first the passive noise model. Fix some $\phi(a)$. Then:
\[
    |\langle \phi(a), \Delta_{\S} \rangle| \leq \sum_{i=1}^{d} |\phi(a)^{(i)}\Delta_{\S}^{(i)}|.
\]
Treating $\phi(a)$ as a fixed vector with $dp$ non-zero entries, the sum is concentrated in the range $dp(\theta \pm \epsilon)$, and so $\rho \leq dp(\theta + \epsilon)$ with high probability. By Chernoff/Hoeffding and the union-bound, with probability $1-\delta$:
\[
    \epsilon \leq \sqrt{\frac{1}{2dp}\ln\frac{2m}{\delta}}.
\]
The result is obtained by noting that $L = \sqrt{pd}$ and applying Theorem \ref{thm:noisy-decoding}.

For the adversarial case, the result is obtained by again observing that $|\langle \phi(a), \Delta_{\S} \rangle| \leq \|\phi(a)\|_{\infty}\|\Delta_{\S}\|_{1} \leq \omega d$ for any $a \in \mathcal{A}$ and applying Theorem \ref{thm:noisy-decoding}.
\end{proof}

\section{Encoding Structures}
\label{section:structure-encoding}

We are often interested in representing more complex data types, such as objects with multiple attributes or ``features.'' In general, we suppose that we observe a set of features $\mathcal{F}$ whose values are assumed to lie in some set $\mathcal{A}$. Let $\psi : \mathcal{F} \rightarrow \mathcal{H}$ be an embedding of features, and $\phi : \mathcal{A} \rightarrow \mathcal{H}$ be an embedding of values. We associate a feature with its value through use of the \emph{binding} operator $\otimes : \mathcal{H} \times \mathcal{H} \rightarrow \mathcal{H}$ that creates an embedding for a (feature,value) pair. For a feature $f \in \mathcal{F}$ taking on a value $a \in \mathcal{A}$, its embedding is constructed as $\psi(f) \otimes \phi(a)$. A data point $\bmm{x} = \{(f_{i} \in \mathcal{F}, x_{i} \in \mathcal{A})\}_{i=1}^{n}$ consists of $n$ such pairs. For simplicity, we assume each $x$ possesses all attributes, although our analysis also applies to the case that $x$ possesses only some subset of attributes. The entire embedding for $\bmm{x}$ is constructed as \shortcite{plate2003holographic}:
\begin{gather}
    \label{eqn:sequence-encoder}
    \phi(\bmm{x}) = \bigoplus_{i=1}^{n} \psi(f_{i}) \otimes \phi(x_{i})
\end{gather}
As with sets we would typically like $\phi(\bmm{x})$ to be \emph{decodable} in the sense that we can recover the value associated with a particular feature, and \emph{comparable} in the sense that $\langle \phi(\bmm{x}), \phi(\bmm{x}') \rangle$ is reflective of a reasonable notion of similarity between $\bmm{x}$ and $\bmm{x}'$.

From a formal perspective, we require the binding operator to satisfy several properties. First, binding should be associative and commutative. That is, for all $\bmm{a},\bmm{b},\bmm{c} \in \mathcal{H}$, $(\bmm{a} \otimes \bmm{b})\otimes \bmm{c} = \bmm{a} \otimes (\bmm{b} \otimes \bmm{c})$, and $\bmm{a} \otimes \bmm{b} = \bmm{b} \otimes \bmm{a}$. Second, there should exist an identity element $\bmm{I} \in \mathcal{H}$ such that $\bmm{I} \otimes \bmm{a} = \bmm{a}$ for all $\bmm{a} \in \mathcal{H}$. Third, for all $\bmm{a} \in \mathcal{H}$, there should exist some $\bmm{a}^{-1}\in \mathcal{H}$ such that $\bmm{a}\otimes\bmm{a}^{-1} = \bmm{I}$. These properties are equivalent to stipulating that $\mathcal{H}$ be an abelian group under $\otimes$. Furthermore, binding should distribute over bundling. That is, for any $\bmm{a}, \bmm{b}, \bmm{c} \in \mathcal{H}$, it should be the case that $\bmm{a} \otimes (\bmm{b} + \bmm{c}) = \bmm{a}\otimes\bmm{b} + \bmm{a} \otimes \bmm{c}$. We here also require that the lengths of bound pairs are bounded, that is to say: $\max_{f \in \mathcal{F}, a\in\mathcal{A}}\|\psi(f) \otimes \phi(a)\|_{2} \leq M$.

A natural choice of embedding satisfying these properties is to sample $\psi(f)$ randomly from $\{\pm 1\}^{d}$ and choose $\otimes$ to be the element-wise product. In this case $\psi(f)$ is its own inverse, that is $\psi(f) \otimes \psi(f) = \bmm{I}$, and binding preserves lengths of codewords. We focus on this case here as it is intuitive, but our analysis generalizes in a straightforward way to any particular implementation satisfying the properties listed above. One can see the bound pairs satisfy various incoherence properties with high probability. For instance, we may declare the binding to be $\mu$-incoherent if:
\[
\underset{a \in \mathcal{A}}{\max}\,\underset{a' \in \mathcal{A}, f \in \mathcal{F}}{\max}\,\langle \phi(a), \psi(f) \otimes \phi(a') \rangle \leq \mu L^{2}
\]
where $L = \min_{a\in\mathcal{A}} \|\phi(a)\|_{2}$. We can extend Theorem \ref{thm:incoherence} to see this property is satisfied with high probability:
\begin{theorem}
\label{thm:id-encoding}
Fix $d, n, m \in \mathbb{Z}^{+}$ and $\mu \in \mathbb{R}^{+}$. Let $\phi$ be centered and $\sigma$-sub-Gaussian, $\otimes$ be the element-wise product, and $\psi(f) \sim \{\pm 1\}^{d}$. Then:
\[
\mathbb{P}(\exists\,a,a' \in \mathcal{A},f \in\mathcal{F} \text{ s.t. } |\langle \phi(a), \phi(a') \otimes \psi(f) \rangle| \ge \mu L^{2}) \le nm^{2}\exp\left( -\frac{\kappa\mu^{2}L^{2}}{2\sigma^{2}} \right)
\]
where $L = \min_{a \in \mathcal{A}} \|\phi(a)\|_{2}$ and $\kappa$ is as defined in Theorem \ref{thm:incoherence}.
\end{theorem}
The proof is similar to Theorem \ref{thm:incoherence} and is available in the Appendix. This result is appealing because it means that the incoherence scales only logarithmically with $m \times n$ which may be large in practice.
As a corollary to the previous theorem, we also obtain the following useful incoherence property:
\begin{gather}
\label{col:id-deocding-incoherence}
    \mathbb{P}(\exists\,a,a',f \neq f' \text{ s.t. } |\langle \phi(a), (\phi(a') \otimes \psi(f)) \otimes \psi^{-1}(f') \rangle| \ge \mu L^{2}) \leq m^{2}n^{2}\exp\left( -\frac{\kappa\mu^{2}L^{2}}{2\sigma^{2}} \right)
\end{gather}
where $\psi^{-1}(f)$ is the inverse of $\psi(f)$ with respect to $\otimes$. This notion of incoherence is useful for decoding representations. Along similar lines: 
\begin{gather}
\label{col:id-incoherence}
    \mathbb{P}(\exists\,a,a',f\neq f' \text{ s.t. } |\langle \phi(a) \otimes \psi(f), \phi(a') \otimes \psi(f') \rangle| \ge \mu L^{2}) \leq m^{2}n^{2}\exp\left( -\frac{\kappa\mu^{2}L^{2}}{2\sigma^{2}} \right)
\end{gather}
We note that the previous statement refers to symbols associated with different attributes and thus does not require any particular incoherence assumption on the $\phi(a)$.

\subsection{Decoding Structures}

This representation can be decoded to recover the value associated with a particular feature. To recover the value of the $i$-th feature, we use the following rule:
\begin{gather*}
    \hat{x}_{i} = \underset{a \in \mathcal{A}}{\text{argmax }} \langle \phi(a), \phi(\bmm{x}) \otimes \psi^{-1}(f_{i}) \rangle
\end{gather*}
where $\psi^{-1}(f)$ denotes the group inverse of $\psi(f)$. Since the binding operator is assumed to distribute over bundling, the dot-product above expands to:
\begin{gather*}
     \langle \phi(a), \phi(x_{i}) \rangle + \sum_{j \ne i} \langle \phi(a), (\phi(x_{j}) \otimes \psi(f_{j})) \otimes \psi^{-1}(f_{i}) \rangle \\
     \begin{cases}
        \geq L^{2}(1 - n\mu) &\text{ if } x_{i} = a \\
        \leq nL^{2}\mu &\text{ otherwise }
     \end{cases}
\end{gather*}
where the incoherence can be bounded as as in Equation \ref{col:id-deocding-incoherence}. Thus $\mu < 1/(2n)$ is a sufficient condition for decodability.

\subsection{Comparing Structures}

As with sets, we may wish to compare two structures without decoding them. As one would expect given Theorem \ref{thm:intersection}, this is can be achieved by computing the dot-product between their encodings:

\begin{theorem}
\label{thm:compare-structures}
Let $\bmm{x}$ and $\bmm{x}'$ be two structures drawn from a common alphabet $\mathcal{F}\times\mathcal{A}$. Denote their encodings using Equation \ref{eqn:sequence-encoder} by $\phi(\bmm{x})$ and $\phi(\bmm{x}')$. Then, if binding is $\mu$-incoherent:
\[
    |\bmm{x} \cap \bmm{x}'| - n^{2}\mu \le \frac{1}{L^{2}}\langle \phi(\bmm{x}),\phi(\bmm{x}') \rangle \le |\bmm{x} \cap \bmm{x}'| + n^{2}\mu
\]
where $\bmm{x} \cap \bmm{x}'$ is defined to be the set $\{i\,:\,x_{i} = x_{i}'\}_{i=1}^{n}$, that is, the features on which $\bmm{x}$ and $\bmm{x}'$ agree.
\end{theorem}
\begin{proof}
Expanding:
\begin{gather*}
    \langle \phi(\bmm{x}),\phi(\bmm{x}') \rangle = \langle \sum_{i=1}^{n} \phi(x_{i})\otimes \psi(f_{i}), \sum_{j=1}^{n} \phi(x_{j}')\otimes\psi(f_{j}) \rangle \\
    = \sum_{i=1}^{n} \langle \phi(x_{i}) \otimes \psi(f_{i}), \phi(x_{i}') \otimes \psi(f_{i}) \rangle + \sum_{i \ne j} \langle \phi(x_{i}) \otimes \psi(f_{i}), \phi(x_{j}') \otimes \psi(f_{j}) \rangle
\end{gather*}
A term in the first sum is $L^{2}$ if $x_{i} = x_{i}'$ and bounded in $\pm L^{2}\mu$ otherwise. So the expression above is bounded as:
\[
    \le L^{2}|\bmm{x} \cap \bmm{x}'| + L^{2}n^{2}\mu
\]
and the other direction of the inequality is analogous.
\end{proof}

\noindent As a practical example, in bioinformatics it is common to search for regions of high similarity between a ``reference'' and ``query'' genome. Work in \shortcite{imani2018hdna} and \shortcite{kim2020genie} explored the use HD computing to accelerate this process by encoding short segments of DNA and estimating similarity on the HD representations.

\subsection{Encoding Sequences}

Sequences are an important form of structured data. In this case, the feature set is simply the list of positions $\{1,2,3,...\}$ in the sequence. In practical applications, we are often interested in streams of data which arrive continuously over time. Typically, real-world processes do not exhibit infinite memory and we only need to store the $n \geq 1$ most recent observations at any time. In the streaming setting, we would like to avoid needing to fully re-encode all $n$ data points each time we receive a new sample, as would be the case using the method described above. 
This motivates the use of shift based encoding schemes \shortcite{kanerva2009hyperdimensional,rahimi2017hyperdimensional,kim2018efficient}. Let $\rho^{(i)}(\bmm{z})$ denote a cyclic left-shift of the elements of $\bmm{z}$ by $i$ coordinates, and $\rho^{(-i)}(\bmm{z})$ denote a cyclic right-shift by $i$ coordinates. In other words: $\rho^{(1)}((z_{1}, z_{2}, \ldots, z_{d-1}, z_{d})) = (z_2, z_3, \ldots, z_d, z_1)$. In shift-based encoding a sequence $\bmm{x} = (x_{1},...,x_{n})$ is represented as:
\[
    \phi(\bmm{x}) = \bigoplus_{i=1}^{n} \rho^{(n-i)}(\phi(x_{i})),
\]
where we take $\oplus$ to be the element wise sum. Now suppose we receive symbol $n+1$ and wish to append it to $\phi(\bmm{x})$ while removing $\phi(x_{1})$. Then we may apply the rule:
\[
    \rho^{(1)}(\phi(\bmm{x}) - \rho^{(n-1)}(\phi(x_{1}))) \oplus \phi(x_{n+1}) = \bigoplus_{i=1}^{n} \rho^{(n-i)}\phi(x_{i+1})
\]
\noindent where we can additionally note that $\rho$ is a special type of permutation and that permutations distribute over sums. However, in order to decode correctly, each $\phi(a)$ must satisfy an incoherence condition with the $\rho^{(j)}(\phi(a'))$. We can again use the randomly generated nature of the codewords to argue this is the case; however, we must here impose the additional restriction that the $\phi(a)$ be bounded, and accordingly restrict attention to the case $\phi(a) \sim \{\pm 1\}^{d}$.
\begin{theorem}
\label{thm:sequence-decoding}
Fix $d,m, n < d \in \mathbb{Z}^{+}$ and $\mu \in \mathbb{R}^{+}$ and let $\phi(a) \sim \{\pm 1\}^{d}$. Then:
\[
    \mathbb{P}(\exists\,a,a' \in \mathcal{A},i\neq 0 \text{ s.t. } |\langle \phi(a), \rho^{(i)}(\phi(a')) \rangle| \geq \mu d) \leq nm^{2}\exp\left( - \frac{\mu^{2}d}{4} \right)
\]
\end{theorem}
\begin{proof}
Fix some $a,a'$ and $i$. In the case that $a \neq a'$, $\phi(a)$ and $\rho^{(i)}(\phi(a))$ are mutually independent. However, when $a = a'$, $\phi(a)$ and $\rho^{(i)}(\phi(a))$ only satisfy pairwise independence and the techniques of Theorem \ref{thm:incoherence} cannot be applied. To resolve this difficulty, let $f(\phi(a)) = \langle \phi(a), \rho^{(i)}(\phi(a)) \rangle$, and denote by $\phi(a)^{\setminus k}$ the vector formed by replacing the $k$-th coordinate in $\phi(a)$ with an arbitrary value $\in \{+1,-1\}$. Then $|f(\phi(a)) - f(\phi(a)^{\setminus k})| \leq 4$ and so by the bounded-differences inequality \shortcite{mcdiarmid1989method}:
\[
    \mathbb{P}(|\langle \phi(a), \rho^{(i)}(\phi(a')) \rangle| \geq \mu d) \leq 2 \exp \left( -\frac{\mu^{2}d}{4} \right).
\]
The result follows by the union bound.
\end{proof}
Several other related methods for encoding sequential information have been proposed in the literature \shortcite{plate2003holographic,gallant2013representing}. For an extensive discussion of these approaches as well as an interesting discussion involving sequences of infinite length, the reader is referred to \shortcite{frady2018theory}.

\subsection{Discussion and Comparison with Prior Work}

We conclude our treatment of encoding and decoding discrete data with some brief discussion of our approach and its relation to antecedents in the literature. A key question addressed here and by several pieces of prior work is to bound the magnitude of crosstalk noise in terms of the encoding dimension ($d$), the number of items to encode ($s$) and the alphabet size ($m$). Early analysis in \shortcite{plate2003holographic,gallant2013representing,kleyko2018classification} recovers the same asymptotic relationship as we do, but only under specific assumptions about the method used to generate the codewords and particular instantiations of the bundling and binding operators. 

Work in \shortcite{frady2018theory} provides a significantly more general treatment which, like ours, aims to abstract away from the particular choice of distribution from which codewords are sampled and from the particular implementation of bundling and binding operator. Their approach assumes the codewords are generated by sampling each component i.i.d. from some distribution and uses the central limit theorem (CLT) to justify modeling the crosstalk noise by a Gaussian distribution. Error bounds in the non-asymptotic setting are then obtained by applying a Chernoff style bound to the resulting Gaussian distribution. This approach again recovers the same asymptotic relationship between $d,s$ and $m$ as us, but does not generally yield formal bounds in the non-asymptotic setting. Our approach based on sub-Gaussianity formalizes this analysis in the non-asymptotic setting. Like us, \shortcite{frady2018theory} also considers the effect of noise on the HD representations, but their treatment is limited to additive white noise, whereas we address both arbitrary additive passive noise and adversarial noise.

In summary, our formalism using the notion of incoherence allows us to decouple the analysis of decoding and noise-robustness from any particular method for generating codewords and readily yields rigorous bounds in the non-asymptotic setting. Our approach is applicable to a large swath of HD computing and enables us to offer more general conditions under which thresholding based decoding schemes will succeed and of the effect of noise than is available in prior work.

\section{Encoding Euclidean Data}
\label{section:euclidean-encoding}

One option for encoding Euclidean vectors is to treat them as a special case of the ``structured data'' considered in the preceding section. As before, we think of our data as a collection of (feature,value) pairs $\bmm{x} = \{(f_{i}, x_{i})\}_{i=1}^{n}$ with the important caveat that $x_{i} \in \mathbb{R}^{n}$. This case is more complex because the feature values may now be continuous, and because the data possesses geometric structure which is typically relevant for downstream tasks and must be preserved by encoding. We here analyze two of the most widely used methods for encoding Euclidean data and discuss general properties of structure preserving embeddings in the context of HD computing.

\subsection{Position-ID Encoding}
\label{subsubsec:position-id}

A widely-used method in practice is to quantize the raw signal to a suitably low precision and then apply the structure encoding method discussed in the previous section \shortcite{rachkovskiy2005sparseScalars,rachkovskiy2005sparseVectors,kleyko2018classification,rahimi2018efficient}.

In this approach, we first quantize the support of each feature $f \in \mathcal{F}$ into some set of $m$ bins with centroids $a_{1} < \cdots < a_{m}$ and assign each bin a codeword $\phi(a) \in \mathcal{H}$. However, instead of requiring the codewords to be incoherent, we now require the correlation between codewords to reflect the distance between corresponding quantizer bins. In other words $\langle \phi(a), \phi(a') \rangle$ should be monotonically decreasing in $|a-a'|$. 

A simple method can be used to generate monotonic codebooks when the codewords are randomly sampled from $\{\pm 1\}^{d}$ \shortcite{rachkovskiy2005sparseScalars,widdows2015reasoning}. Fixing some feature $f$, the codeword for the minimal quantizer bin, $\phi(a_{1})$, is generated by sampling randomly from $\{\pm 1\}^{d}$. To generate the codeword for the second bin, we simply flip some set of $\lceil b \rceil$ bits in $\phi(a_{1})$, where:
\[
    b = \frac{a_{2} - a_{1}}{a_{m} - a_{1}} \cdot \frac{d}{2}
\]
The codeword for the third bin is generated analogously from the second, where we assume the bits to be flipped are sampled such that a bit is flipped at most once. Thus the codewords for the minimal and maximal bins are orthogonal and the correlation between codewords for intermediate bins is monotonically decreasing in the distance between their corresponding bin centroids.

In practice, it seems to be typical to use a single codebook for all features and for the quantizer to be a set of evenly spaced bins over the support of the data. While simple, this approach is likely to have sub-optimal rate when the features are on different scales or are far from the uniform distribution. Encoding then proceeds as follows:
\[
    \phi(\bmm{x}) = \sum_{i=1}^{n} \phi(x_{i}) \otimes \psi(f_{i})
\]
where, as before $\psi \in \{\pm 1\}^{d}$ is a vector which encodes the index $i$ of a feature value $x_{i}$ as in the previous section on encoding sequences; hence the name ``position-ID'' encoding. There are several variations on this theme which are compared empirically in \shortcite{kleyko2018classification}.

This general encoding method was analyzed by \shortcite{rachkovskiy2005sparseVectors}, in the specific case of sparse and binary codewords, who show it preserves the L1 distance between points in expectation but do not provide distortion bounds. We here provide such bounds using our formalism of matrix incoherence. We assume that the underlying quantization of the points is sufficiently fine that it is a low-order term that can be ignored.
\begin{theorem}
\label{thm:id-distances}
Let $\bmm{x}$ and $\bmm{x}'$ be points in $[0,1]^{n}$ with encodings $\phi(\bmm{x})$ and $\phi(\bmm{x}')$ generated using the rule described above. Assume that $\phi$ satisfies $\langle \phi(a),\phi(a') \rangle = d(1 - |a - a'|)$ for all $a,a' \in \mathcal{A}$, and let $\psi \sim \{\pm 1\}^{d}$. Then, for all $\bmm{x},\bmm{x}'$:
\[
2d(\|\bmm{x}-\bmm{x}'\|_{1} - 2n^{2}\mu) \le ||\phi(\bmm{x}) - \phi(\bmm{x}')||^{2}_{2} \le 2d(\|\bmm{x}-\bmm{x}'\|_{1} + 2n^{2}\mu)
\]
\end{theorem}
\noindent The proof is similar to Theorem \ref{thm:compare-structures} and is available in the Appendix.

\noindent The practical implication of the previous theorem is that the position-ID encoding method preserves the L1 distance between points up to an additive distortion which can be bounded by the incoherence of the codebook. Per Equation \ref{col:id-incoherence}, $\mu = O(\sqrt{\ln(mn)/d})$. Therefore, to ensure that $\frac{1}{d}\|\phi(\bmm{x}) - \phi(\bmm{x}')\|_{2}^{2} \approx \|\bmm{x} - \bmm{x}'\|_{1} \pm \epsilon$, the previous result implies we should choose $d = O(\frac{n^{4}}{\epsilon^{2}}\ln(nm))$. This can be relaxed to a quadratic dependence on $n$ in exchange for a weaker pointwise bound, but in either case means the encoding method may be problematic when the dimension of the underlying data is high.

Noting that $||\phi(\bmm{x})||_{2}^{2} \in nd \pm n^{2}d\mu$, we can see that the encodings of each point are roughly of equal norm and lie in a ball of radius at most $n\sqrt{d\mu}$, where the exact position depends on the instantiation of the codebook. Thus, we can loosely interpret the encoding procedure as mapping the data into a thin shell around the surface of a high dimensional sphere.

\subsection{Random Projection Encoding}
\label{section:random-projection-encoding}

Another popular family of encoding methods embeds the data into $\mathcal{H}$ under some random linear map followed by a quantization \shortcite{rachkovskij2015formation,imani2019bric}. More formally, for some $\bmm{x} \in \mathbb{R}^{n}$, these embeddings take the form:
\[
    \phi(\bmm{s}) = g(\bmm{\Phi}\bmm{x})
\]
where $\bmm{\Phi} \in \mathbb{R}^{d \times n}$ is a matrix whose rows are sampled uniformly at random from the surface of the $n$-dimensional unit sphere, and $g$ is a quantizer --- typically the sign function --- restricting the embedding to $\mathcal{H}$. The embedding matrix $\bmm{\Phi}$ may also be quantized to lower precision. This encoding method has also been studied in the context of kernel approximation where it is used to approximate the angular kernel \shortcite{choromanski2017unreasonable}, and to construct low-distortion binary embeddings \shortcite{jacques2013robust,plan2014dimension}. While the following result is well known, we here show this encoding method preserves angular distance up to an additive distortion as this fact is important for subsequent analysis.
\begin{theorem}
\label{thm:rp-dot-pres}
Let $\mathcal{S}^{n-1} \subset \mathbb{R}^{n}$ denote the $n$-dimensional unit sphere. Let $\bmm{\Phi} \in \mathbb{R}^{d\times n}$ be a matrix whose rows are sampled uniformly at random from $\mathcal{S}^{n-1}$. Let $\mathcal{X}$ be a set of points supported on $\mathcal{S}^{n-1}$. Denote the embedding of a point by $\phi(\bmm{x}) = \text{\rm sign}(\bmm{\Phi}\bmm{x})$. Then, for any $\bmm{x},\bmm{x}' \in \mathcal{X}$, with high probability:
\[
    d\theta - O(\sqrt{d}) \leq d_{ham}(\phi(\bmm{x}),\phi(\bmm{x}')) \leq d\theta + O(\sqrt{d})
\]
where $d_{ham}(a,b)$ is the Hamming distance between $a$ and $b$, defined to be the number of coordinates on which $a$ and $b$ differ, and $\theta = \frac{1}{\pi}\cos^{-1}(\langle \bmm{x}, \bmm{x'}\rangle) \in [0,1]$ is proportional to the angle between $\bmm{x}$ and $\bmm{x'}$. 
\end{theorem}
\begin{proof}
Let $\bmm{\Phi}^{(i)}$ denote the $i$th row of the matrix $\bmm{\Phi}$. Then, the $i$th coordinate in the embedding of $\bmm{x}$ can be written as $\text{sign}(\langle \bmm{\Phi}^{(i)}, \bmm{x} \rangle)$. The probability that the embeddings differ on their $i$th coordinate, that is $(\langle \bmm{\Phi}^{(i)}, \bmm{x} \rangle)(\langle \bmm{\Phi}^{(i)},\bmm{x}' \rangle) < 0$, is exactly $\angle(\bmm{x}, \bmm{x}')/\pi$: the angle (in radians) between $\bmm{x}$ and $\bmm{x}'$ divided by $\pi$.

Therefore, the number of coordinates on which $\phi(\bmm{x})$ and $\phi(\bmm{x}')$ disagree is, concentrated in the range, $d(\theta \pm \epsilon)$. By Chernoff/Hoeffding, we have that with probability $1-\delta$:
\[
    d\epsilon \leq \sqrt{2d\ln\frac{2}{\delta}}.
\]
\end{proof}
Noting that $\langle \phi(\bmm{x}), \phi(\bmm{x}') \rangle = d - 2d_{ham}(\phi(\bmm{x}), \phi(\bmm{x}'))$, we obtain the following simple corollary:
\begin{corollary}
Let $\phi$ and $\theta$ be as defined in Theorem \ref{thm:rp-dot-pres}. Then, with high probability:
\[
  d(1-2\theta) - O(\sqrt{d}) \leq \langle \phi(\bmm{x}), \phi(\bmm{x}') \rangle \leq d(1-2\theta) + O(\sqrt{d})  
\]
\end{corollary}
To obtain a more explicit relationship with the dot product, we can use the first-order approximation $\cos^{-1}(x) \approx (\pi/2) - x$, to obtain $\theta \approx \frac{1}{2} - \frac{1}{\pi}\langle \bmm{x}, \bmm{x}' \rangle$, from which we obtain:
\[
     d(1-2\theta) \approx \frac{2d}{\pi}\langle \bmm{x}, \bmm{x}' \rangle.
\]
We emphasize that, in comparison to the position-ID method, the distortion in this case does not depend on the dimension of the underlying data which means this method may be preferable when the data dimension is large.

\subsubsection{Connection with Kernel Approximation}
\label{subsec:hd-kernels}

A natural question is whether the encoding procedure described above, which preserves dot-products, can be generalized to capture more diverse notions of similarity? We can answer in the affirmative by noting that the random projection encoding method is closely related to the notion of random Fourier features which have been widely used for kernel approximation \shortcite{rahimi2008random}. The basic idea is to construct an embedding $\phi : \mathbb{R}^{n} \rightarrow \mathbb{R}^{d}$, such that $\langle \phi(\bmm{x}), \phi(\bmm{x}') \rangle \approx k(\bmm{x},\bmm{x}')$, where $k$ is a shift-invariant kernel. The construction exploits the fact that the Fourier transform of a shift-invariant kernel $k$ is a probability measure: a well known result from harmonic analysis known as Bochner's Theorem \shortcite{rudin1962fourier}. The embedding itself is given by $\phi(\bmm{x}) = \frac{1}{\sqrt{d}}\cos(\bmm{\Phi}\bmm{x} + \bmm{b})$, where the rows of $\bmm{\Phi}$ are sampled from the distribution induced by $k$ and the coordinates of $\bmm{b}$ are sampled uniformly at random from $[0,2\pi]$.

Subsequent work in \shortcite{raginsky2009locality} gave a simple scheme for quantizing the embeddings produced from random Fourier features to binary precision. Their construction yields an embedding $\psi : \mathbb{R}^{n} \rightarrow \{0,1\}^{d}$ such that:
\[
    f_{1}(k(\bmm{x},\bmm{x}')) - \Delta \leq \frac{1}{d}d_{ham}(\psi(\bmm{x}), \psi(\bmm{x}')) \leq f_{2}(k(\bmm{x},\bmm{x}')) + \Delta
\]
where $f_{1},f_{2} : \mathbb{R} \rightarrow \mathbb{R}$ are independent of the choice of kernel, and $\Delta$ is a distortion term. The embedding itself is constructed by applying a quantizer $Q_{t}(x) = \text{sign}(x + t)$ coordinate wise over the embeddings constructed from random Fourier features. In other words $\psi(\bmm{x})_{i} = \frac{1}{2}(1 + Q_{t_{i}}(\phi(\bmm{x})_{i}))$, where $t_{i} \sim \text{Unif}[-1,1]$, and $\phi(\bmm{x})$ is a random Fourier feature.

This connection is highly appealing for HD computing. The quantized random Fourier feature scheme presents a simple recipe for constructing encoding methods meeting the desiderata of HD computing while preserving a rich variety of structure in data. For instance, shift-invariant kernels preserving the L1 and L2 distance---among many others---can be approximated using the method discussed above. Furthermore, this observation provides a natural point of contact between HD computing and the vast literature on kernel methods which has produced a wealth of algorithmic and theoretical insights.

\subsection{Consequences of Distance Preservation}

The encoding methods discussed above are both appealing because they preserve reasonable notions of distance between points in the original data. Distance preservation is a sufficient condition to establish other desirable properties of encodings, namely preservation of neighborhood/cluster structure, robustness to various forms of noise, and in some cases, preservation of linear separability. We address the first two items here and defer the latter for our discussion of learning on HD representations. We formalize our notion of distance preservation as follows:
\begin{definition}
\label{def:distance-preservation}
\textbf{Distance-Preserving Embedding}: Let $\delta_{\mathcal{X}}$ be a distance function on $\mathcal{X} \subset \mathbb{R}^{n}$ and $\delta_{H}$ be a distance function on $\mathcal{H}$. We say $\phi$ preserves $\delta_{\mathcal{X}}$ under $\delta_{H}$ if, there exist functions $\alpha,\beta : \mathbb{Z}^{+} \rightarrow \mathbb{R}$ such that $\beta(d)/\alpha(d)\rightarrow 0$ as $d \to \infty$, and:
\begin{gather}
    \label{eqn:distance-preservation}
    \alpha(d)\delta_{\mathcal{X}}(\bmm{x},\bmm{x}') -\beta(d) \leq \delta_{\mathcal{H}}(\phi(\bmm{x}), \phi(\bmm{x}')) \leq \alpha(d)\delta_{\mathcal{X}}(\bmm{x},\bmm{x}') + \beta(d)
\end{gather}
for all $\bmm{x},\bmm{x}' \in \mathcal{X}$.
\end{definition}
We typically wish the distance function $\delta_{\mathcal{H}}$ on $\mathcal{H}$ to be simple to compute. In practice, it is often taken to be the Euclidean, Hamming, or angular distance. The position-ID method preserves the L1 distance with $\delta_{H}$ the squared Euclidean distance, $\alpha(d) = 2d$, and $\beta(d) \leq n^{2}\mu d$; recall that in the constructions above, $\mu$ scales inversely with $d$ and thus $\beta(d)/\alpha(d) \to 0$. The signed random-projection method preserves the angular distance with $\alpha(d) = O(d)$, $\beta(d) = O(\sqrt{d})$, and $\delta_{H}$ the Hamming, angular, or Euclidean distance.

\subsubsection{Preservation of Cluster Structure}

In general, there is no universally applicable definition of cluster structure. Indeed, numerous algorithms have been proposed in the literature to target various reasonable notions of what constitutes a ``cluster'' in the data. Preservation of a distance function accords naturally with K-means like algorithms which, given a set of data $\mathcal{X} \subset \mathbb{R}^{n}$ compute a set of centroids $\mathcal{C} = \{\bmm{c}_{i}\}_{i=1}^{k}$, and define associated clusters as the Voronoi cells associated with each centroid. We here adopt this notion and state that cluster structure $\mathcal{C}$ is preserved if, for any $\bmm{x} \in \mathcal{X}$:
\[
    \underset{\bmm{c} \in \mathcal{C}}{\text{argmin }} \delta_{\mathcal{X}}(\bmm{x},\bmm{c}) = \underset{\bmm{c} \in \mathcal{C}}{\text{argmin }} \delta_{\mathcal{H}}(\phi(\bmm{x}),\phi(\bmm{c}))
\]
In other words, that the set of points bound to a particular cluster centroid does not change under the encoding. We can restate the above as requiring that, for some point $\bmm{x}$ bound to a cluster centroid $\bmm{c}$, it is the case that:
\[
    \delta_{\mathcal{H}}(\phi(\bmm{x}), \phi(\bmm{c})) < \delta_{\mathcal{H}}(\phi(\bmm{x}), \phi(\bmm{c'}))
\]
for any $\bmm{c}' \in \mathcal{C}\setminus\{\bmm{c}\}$. From Definition \ref{def:distance-preservation} we have:
\[
\delta_{\mathcal{H}}(\phi(\bmm{x}), \phi(\bmm{c'})) - \delta_{\mathcal{H}}(\phi(\bmm{x}), \phi(\bmm{c})) \geq \alpha(d)(\delta_{\mathcal{X}}(\bmm{x}, \bmm{c'}) - \delta_{\mathcal{X}}(\bmm{x}, \bmm{c})) - 2 \beta(d)
\]
for any $\bmm{x} \in \mathcal{X}$ and $\bmm{c},\bmm{c}' \in \mathcal{C}$. Rearranging the expressions above we can see the desired property will be satisfied if:
\[
    \frac{\beta(d)}{\alpha(d)} < \underset{\bmm{x} \in \mathcal{X}}{\min}\,\underset{\bmm{c}' \neq \bmm{c}(\bmm{x})}{\min}\,\frac{1}{2}(\delta_{\mathcal{X}}(\bmm{x},\bmm{c}') - \delta_{\mathcal{X}}(\bmm{x},\bmm{c}(\bmm{x}))),
\]
where $\bmm{c}(\bmm{x}) = {\text{argmin}}_{\bmm{c} \in \mathcal{C}} \delta_{\mathcal{X}}(\bmm{x}, \bmm{c})$ denotes the center in $\mathcal{C}$ closest to $\bmm{x}$.
A sufficient condition for the existence of some $d$ satisfying this property is that $\alpha(d)$ is monotone increasing and that $\alpha(d)$ is faster growing than $\beta(d)$. This condition is satisfied for both the random projection and position-ID encoding methods.

\subsubsection{Noise Robustness}

It is also of interest to consider robustness to noise in the context of encoding Euclidean data. Suppose we have a set of points, $\mathcal{X}$, in $\mathbb{R}^{n}$, and a distance function of interest $\delta_{\mathcal{X}}(\cdot,\cdot)$ which is preserved \emph{\`{a} la} Definition \ref{def:distance-preservation}. Given an arbitrary point $\bmm{x} \in \mathcal{X}$ we consider a noise model which corrupts $\phi(\bmm{x})$ to $\phi(\bmm{x}) + \Delta$, where $\Delta$ is some unspecified noise process. Along the lines of Section \ref{section:discrete-robustness}, we say $\Delta$ is $\rho$-bounded if:
\[
    \underset{\bmm{x} \in \mathcal{X}}{\max}\,|\langle \phi(\bmm{x}), \Delta \rangle| \leq \rho
\] 
Suppose we wish to ensure the encodings can distinguish between all points at a distance $\leq \epsilon_{1}$ from $\bmm{x}$ and all points at a distance $\geq \epsilon_{2}$. That is:
\[
    \|\phi(\bmm{x}) + \Delta - \phi(\bmm{x}')\| < \|\phi(\bmm{x}) + \Delta - \phi(\bmm{x}'')\|
\]
for all $\bmm{x}' \in \mathcal{X}$ such that $\delta_{\mathcal{X}}(\bmm{x},\bmm{x}') \leq \epsilon_{1}$ and all $\bmm{x}'' \in \mathcal{X}$ such that $\delta_{\mathcal{X}}(\bmm{x},\bmm{x}') \geq \epsilon_{2}$. We say that such an encoding is $(\epsilon_{1}, \epsilon_{2})$-robust.

\begin{theorem}
Let $\delta_{\mathcal{X}}$ be a distance function on $\mathcal{X} \subset \mathbb{R}^{n}$ and suppose $\phi$ is an embedding preserving $\delta_{\mathcal{X}}$ under the squared Euclidean distance on $\mathcal{H}$ as described in Definition \ref{def:distance-preservation}. Suppose $\Delta$ is $\rho$-bounded noise. Then $\phi$ is $(\epsilon_{1},\epsilon_{2})$ robust if:
\[
    \rho < \frac{\alpha(d)}{4}(\epsilon_{2} - \epsilon_{1}) - \frac{\beta(d)}{2} .
\]
\end{theorem}
\begin{proof}
Fix a point $\bmm{x}$ whose encoding is corrupted as $\phi(\bmm{x}) + \Delta$. Then for any $\bmm{x}', \bmm{x}'' \in \mathcal{X}$ with $\delta_{\mathcal{X}}(\bmm{x},\bmm{x}') \leq \epsilon_{1}$ and $\delta_{\mathcal{X}}(\bmm{x},\bmm{x}'') \geq \epsilon_{2}$, we have:
\begin{align*}
\lefteqn{\|\phi(\bmm{x}) + \Delta - \phi(\bmm{x}'')\|_{2}^{2} - \|\phi(\bmm{x}) + \Delta - \phi(\bmm{x}')\|_{2}^{2} } \\
&=
 \|\phi(\bmm{x}) - \phi(\bmm{x}'')\|_{2}^{2} - \|\phi(\bmm{x}) - \phi(\bmm{x}')\|_{2}^{2} - 2\langle \phi(\bmm{x}''), \Delta \rangle + 2\langle \phi(\bmm{x}'), \Delta \rangle \\
&\geq \alpha(d) \delta_{\mathcal{X}}(\bmm{x},\bmm{x}'') - \beta(d) - \alpha(d) \delta_{\mathcal{X}}(\bmm{x},\bmm{x}') - \beta(d) - 4 \rho \\
&\geq \alpha(d) (\epsilon_2 - \epsilon_1) - 2 \beta(d) - 4 \rho \ > \ 0,
\end{align*}
as desired.
\end{proof}

\noindent As before, we may consider passive and adversarial examples.

\vspace{2mm}
\noindent \textbf{Additive White Gaussian Noise}. First consider the case that $\mathcal{H} = \mathbb{R}^{d}$ and $\Delta \sim \mathcal{N}(0,\sigma_{\Delta}^{2}\bmm{I}_{d})$; that is, each coordinate of $\Delta$ has a Gaussian distribution with mean zero and variance $\sigma_\Delta^2$. Then, as before, we can note that $\langle \phi(\bmm{x}), \Delta \rangle \sim \mathcal{N}(0,\sigma_{\Delta}^{2}\|\phi(\bmm{x})\|_{2}^{2})$. Then, it is very likely (four standard deviations in the tail of the normal distribution) that $\rho < 4L\sigma_{\Delta}$, where $L = \max_{\bmm{x} \in \mathcal{X}} \|\phi(\bmm{x})\|_{2}$. So then, we have the desired robustness property if:
\[
    \sigma_{\Delta} < \frac{\alpha(d)}{16L}(\epsilon_{2} - \epsilon_{1}) - \frac{\beta(d)}{8L}
\]
Assuming that $\alpha(d)$ is faster growing in $d$ than $L$ and $\beta(d)$, there will exist some encoding dimension for which we can tolerate any given level of noise. In the case of the random projection encoding scheme described above $\alpha(d) = O(d), \beta(d) = O(\sqrt{d})$ and $L = \sqrt{d}$ exactly. And so we can tolerate noise on the order of:
\[
    \sigma_{\Delta} \approx \sqrt{d} \, (\epsilon_{2} - \epsilon_{1}) - O(1)
\]
For the position-ID encoding method, $\alpha(d) = O(d)$, $L = O(\sqrt{nd})$ and $\beta(d) = O(n^{2}d\mu)$, and so we can tolerate noise:
\[
    \sigma_{\Delta} \approx \sqrt{\frac{d}{n}}((\epsilon_{2} - \epsilon_{1}) - O(n^{2}\mu))
\]

\vspace{2mm}
\noindent\textbf{Adversarial Noise}. We now consider the case that $\mathcal{H} = \{\pm 1\}$, as in the random-projection encoding method, and $\Delta$ is noise in which some fraction $\omega \cdot d$ of coordinates in $\phi(\bmm{x})$ are maliciously corrupted by an adversary. Since $\|\Delta\|_{1} \leq \omega d$, we have, for any $\bmm{x} \in \mathcal{X}$:
\[
    |\langle \phi(\bmm{x}), \Delta \rangle| \leq \|\phi(\bmm{x})\|_{\infty}\|\Delta\|_{1} \leq \omega d
\]
So then we can tolerate $\omega$ on the order of:
\[
    \omega < \frac{\alpha(d)}{4d}(\epsilon_{2} - \epsilon_{1}) - \frac{\beta(d)}{2d}
\]
In the case of the random-projection encoding method this boils down to:
\[
    \omega \approx (\epsilon_{2} - \epsilon_{1}) - \frac{1}{\sqrt{d}},
\]
meaning the total number of coordinates that can be corrupted is $O(d(\epsilon_{2} - \epsilon_{1}))$.

\vspace{2mm}
\noindent\textbf{Robustness to Input Noise}. A natural question is whether the HD representations also confer any robustness to noise in the input space $\mathcal{X}$ rather than the HD space $\mathcal{H}$. In general, preservation of distance does not imply any particular robustness to input noise and the answer to this question depends on the particulars of the encoding method in question. Since a general treatment is difficult to give, we will not pursue this matter in depth at present. 

\section{Learning on HD Data Representations}
\label{section:learning}

We now turn to the question of using HD representations in learning algorithms. Our goal is to clarify in what precise sense the HD encoding process can make learning easier. We study two ways in which this can happen: the encoding process can increase the separation between classes and/or can induce sparsity. Both of these characteristics can be exploited by neurally plausible algorithms to simplify learning. Throughout this discussion, we assume access to a set of $N$ labelled examples $\mathcal{S} = \{(\bmm{x}_{i}, y_{i})\}_{i=1}^{N}$, where $\bmm{x}_{i}$ lies in $[0,1]^{n}$ and $y_{i} \in \mathcal{C}$ is a categorical variable indicating the class label. In general, we are interested in the case that training examples arrive in a streaming, or online, fashion, although our conclusions apply to fixed and finite data as well.

\subsection{Learning by Bundling}

The simplest approach to learning with HD representations is to bundle together the training examples corresponding to each class into a set of exemplars---often referred to as ``prototypes''---which are then used for classification \shortcite{kleyko2018classification,rahimi2018efficient,burrello2018one}. More formally, as described in Section \ref{section:background}, we construct the prototype $\bmm{c}_{k}$ for the k-th class as:
\[
    \bmm{c}_{k} = \bigoplus_{i \text{ s.t. } y_{i} = k} \phi(\bmm{x}_{i})
\]
and then assign a class label for some ``query'' point $\bmm{x}_{q}$ as:
\begin{gather}
    \label{eqn:basic-classifier}
    \hat{y} = \underset{k \in \mathcal{C}}{\text{argmax}}\, \frac{\langle \bmm{c}_{k}, \phi(\bmm{x}) \rangle}{||\bmm{c}_{k}||}
\end{gather}
This approach bears a strong resemblance to naive Bayes and Fisher's linear discriminant, which are both classic simple statistical procedures for classification \shortcite{bishop2006pattern}. Like these methods, the bundling approach is appealing due to its simplicity. However, it also shares their weaknesses in that it may fail to separate data that is in fact linearly separable.

\subsection{Learning Arbitrary Linear Separators}

Linear separability is one of the most basic types of structure that can aid learning. The theory of linear models is well developed and several simple, neurally plausible, algorithms for learning linear separators are known, for instance, the Perceptron and Winnow \shortcite{rosenblatt.58,littlestone1988learning}. Thus, if our data is linearly separable in low-dimensional space we would like it to remain so after encoding, so that these methods can be applied. We now show formally that preservation of distance is sufficient, under some conditions, to preserve linear separability.

\begin{theorem}
Let $\mathcal{X}$ and $\mathcal{X}'$ be two disjoint, closed, and convex sets of points in $\mathbb{R}^{n}$. Let $\bmm{p} \in \mathcal{X}$ and $\bmm{q} \in \mathcal{X}'$ be the closest pair of points between the two sets. Suppose $\phi$ preserves L2 distance on $\mathcal{X}$ under the  L2 distance on $\mathcal{H}$ in the sense of Definition \ref{def:distance-preservation}. Then, the function $f(\bmm{x}) = \langle \phi(\bmm{x}), \phi(\bmm{p}) - \phi(\bmm{q}) \rangle - \frac{1}{2}(||\phi(\bmm{p})||_{2}^{2} - ||\phi(\bmm{q})||_{2}^{2})$ is positive for all $\bmm{x} \in \mathcal{X}$ and negative for all $\bmm{x}' \in \mathcal{X}'$ provided:
\begin{gather*}
        \frac{\beta(d)}{\alpha(d)} < \frac{1}{2}\|\bmm{p} - \bmm{q}\|_{2}^{2} .
\end{gather*}
\end{theorem}
\begin{proof}
We first observe:
\begin{gather*}
\langle \phi(\bmm{x}), \phi(\bmm{p}) - \phi(\bmm{q}) \rangle - \frac{1}{2}\left( \|\phi(\bmm{p})\|_{2}^{2} - \|\phi(\bmm{q})\|_{2}^{2} \right)
 = \frac{1}{2} \|\phi(\bmm{x}) - \phi(\bmm{q})\|_{2}^{2} - \frac{1}{2} \|\phi(\bmm{x}) - \phi(\bmm{p})\|_{2}^{2}.
\end{gather*}
We may then use Definition \ref{def:distance-preservation} to obtain:
\begin{align*}
f(\bmm{x}) &= \frac{1}{2} \|\phi(\bmm{x}) - \phi(\bmm{q})\|_{2}^{2} - \frac{1}{2} \|\phi(\bmm{x}) - \phi(\bmm{p})\|_{2}^{2} \\
&\geq
\frac{\alpha(d)}{2} \|\bmm{x} - \bmm{q}\|_{2}^{2} - \frac{\alpha(d)}{2} \|\bmm{x} - \bmm{p}\|_{2}^{2} - \beta(d) \\
&=
\alpha(d) \left( \langle \bmm{x}, \bmm{p} - \bmm{q} \rangle - \frac{1}{2}\left( \|\bmm{p}\|_{2}^{2} - \|\bmm{q}\|_{2}^{2} \right) \right) - \beta(d).
\end{align*}
By a standard proof of the hyperplane separation theorem (e.g., Section 2.5.1 of \shortcite{boyd2004convex}), 
\[
    \langle \bmm{x}, \bmm{p} - \bmm{q} \rangle - \frac{1}{2} (\|\bmm{p}\|_{2}^{2} - \|\bmm{q}\|_{2}^{2}) \geq \frac{1}{2}\|\bmm{p} - \bmm{q}\|_{2}^{2}
\]
for any $\bmm{x} \in \mathcal{X}$, and thus $f(\bmm{x}) > 0$ if 
\begin{gather*}
    \frac{\beta(d)}{\alpha(d)} < \frac{1}{2}\|\bmm{p} - \bmm{q}\|_{2}^{2} .
\end{gather*}
the proof for $\bmm{x} \in \mathcal{X}'$ is analogous. 
\end{proof}
A natural question is whether a linear separator on the HD representation can capture a \emph{nonlinear} decision boundary on the original data? The connection with kernel methods discussed in Section \ref{subsec:hd-kernels} presents one avenue for rigorously addressing this question. As noted there, the encoding function can sometimes be interpreted as approximating the feature map of a kernel, which in turn can be used to linearize learning problems in some settings \shortcite{shawe2004kernel}. However, a thorough examination of this question is beyond the scope of the present work.

\subsubsection{Learning Sparse Classifiers on Random Projection Encodings}

The random projection encoding method can be seen to lead to representations that are \emph{sparse} in the sense that a subset of just $k \ll d$ coordinates suffice for determining the class label. This setting accords naturally with the Winnow algorithm \shortcite{littlestone1988learning} which is known to make on the order of $k \log d$ mistakes when the target function class is a linear function of $k \leq d$ variables. This can offer substantially faster convergence than the Perceptron when the margin is small. Curiously, while the Perceptron algorithm is commonly used in the HD community, we are unaware of any work using Winnow for learning. 

\begin{theorem}
\label{thm:random-projections}
Let $\mathcal{X}$ and $\mathcal{X}'$ be two sets of points supported on the $n$-dimensional unit sphere and separated by a unit-norm hyperplane $\bmm{w}$ with margin $\gamma = \min_{\bmm{x}\in\mathcal{X}} |\langle \bmm{x},\bmm{w} \rangle|$. Let $\bmm{\Phi} \in \mathbb{R}^{d \times n}$ be a matrix whose rows are sampled from the uniform distribution over the $n$-dimensional unit-sphere. 
Define the encoding of a point $\bmm{x}$ by $\phi(\bmm{x}) = \bmm{\Phi}\bmm{x}$. 
With high probability, $\mathcal{X}$ and $\mathcal{X}'$ are linearly separable using just $k$ coordinates in the encoded space, provided:
\[
    d = \Omega\left(k\exp\left( \frac{n}{2k\gamma^{2}} \right)\right).
\]
\end{theorem}

\noindent To prove the theorem we first use the following simple Lemma:
\begin{lemma}
Suppose there exists a row $\bmm{\Phi}^{(i)}$ of the projection matrix such that $\langle \bmm{\Phi}^{(i)}, \bmm{w} \rangle > 1 - \gamma^2/2$. Then $\langle \bmm{\Phi}^{(i)}, \bmm{x} \rangle$ is positive for any $\bmm{x} \in \mathcal{X}$ and negative for any $\bmm{x} \in \mathcal{X}'$.
\end{lemma}
\begin{proof}
The constraint on the dot product of $\bmm{\Phi}^{(i)}$ and $\bmm{w}$ implies $\|\bmm{\Phi}^{(i)} - \bmm{w}\|^2 = \|\bmm{\Phi}^{(i)}\|^2 + \|\bmm{w}\|^2 - 2 \langle \bmm{\Phi}^{(i)},  \bmm{w} \rangle < \gamma^2$. Thus for any $\bmm{x} \in \mathcal{X}$,
\[
    \langle \bmm{\Phi}^{(i)},\bmm{x} \rangle 
    = \langle \bmm{w}, \bmm{x} \rangle + \langle \bmm{\Phi}^{(i)} - \bmm{w},\bmm{x} \rangle 
    \ge \gamma + \langle \bmm{\Phi}^{(i)} - \bmm{w}, \bmm{x} \rangle
    \geq \gamma - \|\bmm{\Phi}^{(i)} - \bmm{w}\|
    > 0 .
\]
A similar argument shows that $\langle \bmm{\Phi}^{(i)},\bmm{x} \rangle$ is negative on $\mathcal{X}'$.
\end{proof}
\noindent Unfortunately, the probability of randomly sampling such a direction is tiny, on the order of $\gamma^{n}$.
However, we might instead hope to sample $k$ vectors that are weakly correlated with $\bmm{w}$ and exploit their cumulative effect on $\bmm{x}$. We say a vector $\bmm{u} \in \mathbb{R}^{n}$ is $\rho$-correlated with $\bmm{w}$ if $\langle \bmm{u}, \bmm{w} \rangle \ge \rho$. We are now in a position to prove the theorem.
\vspace{2mm}
\begin{proof}
For $\bmm{w} \in \mathcal{S}^{n-1}$ and $\rho \in (0,1)$, let $\mathcal{C} = \{\bmm{u} \in S^{n-1}: \langle \bmm{u}, \bmm{w} \rangle\geq \rho\}$ denote the spherical cap of vectors $\rho$-correlated with $\bmm{w}$. Suppose we pick vectors $\bmm{u}^{(1)}, \ldots, \bmm{u}^{(k)}$ uniformly at random from $\mathcal{C}$. Then, with probability at least $1/2$:
\begin{gather} 
    \label{eqn:rhosum}
    \frac{\langle \sum_{j} \bmm{u}^{(j)}, \bmm{w} \rangle}{\|\sum_{j} \bmm{u}^{(j)}\|_{2}} \geq 1 - \frac{1}{2k\rho^{2}}
\end{gather}
To see this, note that without loss of generality we may assume $\bmm{w} = \bmm{e}_{1}$, the first standard basis vector of $\mathbb{R}^{n}$, and write any $\bmm{u} \in \mathbb{R}^{n}$ as $\bmm{u} = (u_{1}, \bmm{u}_{R})$: the first coordinate and the remaining $n-1$ coordinates. Now, let $N = \langle \sum_{j} \bmm{u}^{(j)}, \bmm{w} \rangle = \sum_{j} \bmm{u}^{(j)}_{1} \geq k\rho$. Then:
\begin{align*}
    \bigg\| \sum_{j} \bmm{u}^{(j)} \bigg\|_{2}^{2} &= \left(\sum_{j} \bmm{u}_{1}^{(j)}\right)^{2} + \bigg\|\sum_{j} \bmm{u}_{R}^{(j)}\bigg\|_{2}^{2} \\
    &= N^{2} + \sum_{j} \|\bmm{u}_{R}^{(j)}\|_{2}^{2} + \sum_{i\neq j} \langle \bmm{u}_{R}^{(i)}, \bmm{u}_{R}^{(j)} \rangle \\
    &\leq N^{2} + k + \sum_{i\neq j} \langle \bmm{u}_{R}^{(i)}, \bmm{u}_{R}^{(j)} \rangle.
\end{align*}
The last term has a symmetric distribution around zero over random samplings of the $\bmm{u}^{(j)}$. Thus, with probability $\geq 1/2$, it is $\leq 0$, whereupon
\[
    \frac{\langle \sum_{j} \bmm{u}^{(j)}, \bmm{w} \rangle}{\|\sum_{j} \bmm{u}^{(j)}\|_{2}} \geq \frac{N}{\sqrt{N^2 + k}} 
\geq 1 - \frac{k}{2N^2}
\geq 1 - \frac{1}{2k\rho^2}.
\]
To ensure the quantity above is at least $1 - \gamma^{2}/2$, we must have:
\[
    \rho^{2} \geq \frac{1}{k\gamma^{2}}.
\]
It now remains to compute the probability that a vector $\bmm{\Phi}^{(i)}$ sampled uniformly from $\mathcal{S}^{n-1}$ lies in $\mathcal{C}$, or equivalently, that $\bmm{\Phi}_{1}^{(i)} \geq \rho$. Noting that we may simulate a random direction on $\mathcal{S}^{n-1}$ by sampling $\bmm{z} \sim \mathcal{N}(0,\bmm{I}_{n})$ and normalizing, we obtain the reasonable approximation: $\bmm{\Phi}^{(i)}_{1} \sim \mathcal{N}(0,1/n)$. Therefore, the probability that $\bmm{\Phi}_{1}^{(i)} \geq \rho$ is on the order of $e^{-n\rho^{2}/2}$. So we need:
\[
    d = \Omega\left(k\exp\left( \frac{n}{2k\gamma^{2}} \right)\right)
\]
\end{proof}
In summary, the random projection method in tandem with the Winnow algorithm seems to be well suited to the HD setting, where sparsity can be exploited to simplify learning.

\section{Conclusion}

To conclude, we lay out several research directions related to HD computing we believe it would be of particular interest to further explore. There are several interesting open problems related to encoding. Our analysis established preservation of only the most basic forms of structure in data. Can encoding procedures satisfying the desiderata of HD computing be designed that capture other forms of structure? The quantized random Fourier feature construction discussed in Section \ref{section:euclidean-encoding} presents one such option, but is only applicable to structure that can be captured using a shift-invariant kernel on a Euclidean space. For instance, can we devise encoding methods that exploit low-dimensional manifold structure in the data or which are adaptive and can be learned from a particular data set?

Several recent works have claimed, based on empirical evidence, that HD computing evinces one-shot learning \shortcite{burrello2018one,imani2017voicehd,rahimi2018efficient} in which a single labeled example is needed to learn a generalizable classifier \shortcite{thrun1996learning,lake2011one}. However, this work has focused on settings in which specialized hand-crafted features could be extracted, and it is not clear to us that existing encoding procedures would lead to one-shot classifiers absent such outside information. We would be interested to explore whether the HD representation makes one-shot learning easier in any broader sense. We expect this will necessitate the use of more sophisticated encoding procedures that can learn salient properties of a given domain. For this latter point we see dictionary learning \shortcite{olshausen.field.96} as a promising avenue for developing adaptive encoding procedures. Dictionary learning is a well studied problem and can be solved using online and neurally plausible methods \shortcite{arora2015simple,mairal2010online} and would thus seem to be a promising avenue to address the limitations of existing encoding procedures without sacrificing the simplicity and neural plausibility of existing HD based methods.

\section*{Acknowledgements}

This work was supported in part by CRISP, one of six centers in JUMP, an SRC program sponsored by DARPA, in part by an SRC-Global Research Collaboration grant, GRC TASK 3021.001, GRC TASK 2942.001, DARPA-PA-19-03-03 Agreement HR00112090036, and also NSF grants 1527034, 1730158, 1826967, 2100237, 2112167, 2052809, 2003279, 1830399, 1911095, 2028040, and 1911095. 

\section*{Appendix A. Proofs of Selected Theorems}

\input{proofs}

\vskip 0.2in
\typeout{}
\bibliographystyle{theapa}
\bibliography{sample}

\end{document}

%% file: proofs.tex
\subsection*{A.1 Proof of Theorem 4}

\begin{proof}
The result is an immediate consequence of the Hanson-Wright inequality \shortcite{hanson1971bound,rudelson2013hanson} which holds that, for $\bmm{x}$ a centered, $d$-dimensional, $\sigma$-sub-Gaussian random vector, and $\bmm{A} \in \mathbb{R}^{d \times d}$ an arbitrary square matrix, the quadratic form $\bmm{x}^{T}\bmm{A}\bmm{x}$ obeys the following concentration bound:
\[
    \mathbb{P}(|\bmm{x}^{T}\bmm{A}\bmm{x} - \mathbb{E}[\bmm{x}^{T}\bmm{A}\bmm{x}]| \geq t) \leq 2\exp\left(-c\min\left( \frac{t^{2}}{\sigma^{4}\|\bmm{A}\|_{F}^{2}}, \frac{t}{\sigma^{2}\|\bmm{A}\|} \right)\right)
\]
where $c$ is a positive absolute constant, $\|\bmm{A}\|_{F}^{2} = \sum_{i,j} |\bmm{A}_{ij}|^{2}$ is the Frobenius norm and $\|\bmm{A}\| = \max_{\|\bmm{x}\| \leq 1}\|\bmm{A}\bmm{x}\|$ is the operator norm. The result follows by taking $\bmm{A}$ to be the $d \times d$ identity matrix, in which case $\bmm{x}^{T}\bmm{I}_{d}\bmm{x} = \|\bmm{x}\|_{2}^{2}$, and union bounding over all $m$ symbols in the alphabet.
\end{proof}

\subsection*{A.2 Proof of Theorem 7}

\begin{proof}
Fix some $a \notin \mathcal{S}$. As described in Theorem \ref{thm:incoherence}, the quantity $\langle \phi(a), \phi(a') \rangle$ is sub-Gaussian with parameter at most $L_{\max}^{2}\sigma^{2}$, where $L_{\max} = \max_{a} \|\phi(a)\|$. Then, again using the fact that sub-Gaussianity is preserved under sums, by Hoeffding's inequality we have:
\[
    \mathbb{P}\left(\left| \sum_{a' \in \S} \langle \phi(a), \phi(a') \rangle \right| \geq \tau L^{2}\right) \leq 2\exp\left( -\frac{\tau^{2}L^{4}}{2sL_{\max}^{2}\sigma^{2}} \right) \leq 2\exp\left( -\frac{\kappa\tau^{2}L^{2}}{2s\sigma^{2}} \right)
\]
where $\kappa = L^{2}/L_{\max}^{2}$. The result follows by union bounding over all $m$ possible $a$.
\end{proof}

\subsection*{A.3 Proof of Theorem 9}

\begin{proof}
Expanding the dot product between the two representations:
\begin{align*}
    \frac{1}{L^{2}}\langle \phi(\mathcal{S}), \phi(\mathcal{S}') \rangle &= \frac{1}{L^{2}}\sum_{a\in \mathcal{S} \cap \mathcal{S}'} \langle \phi(a), \phi(a) \rangle + \frac{1}{L^{2}}\sum_{a \in \mathcal{S}} \sum_{a' \in \mathcal{S}'\setminus\{a\}}\langle \phi(a), \phi(a')\rangle \\
    &\leq |\mathcal{S} \cap \mathcal{S}'| + ss'\mu .
\end{align*}
The other direction is analogous.
\end{proof}

\subsection*{A.4 Proof of Theorem 10}

\begin{proof}
Consider some symbol $a \in \mathcal{A}$. In the event $a \in \mathcal{S}$:
\[
    \langle \phi(a), \phi(\mathcal{S}) + \Delta_{\S} \rangle = \langle \phi(a), \phi(\mathcal{S}) \rangle + \langle \phi(a), \Delta_{\mathcal{S}} \rangle \geq L^{2} - sL^{2}\mu - \rho
\]
and when $a \notin \mathcal{S}$:
\[
    \langle \phi(a), \phi(\mathcal{S}) + \Delta_{\S} \rangle \leq sL^{2}\mu + \rho
\]
Therefore we can decode correctly if:
\[
    \frac{\rho}{L^{2}} + s\mu < \frac{1}{2}
\]
\end{proof}

\subsection*{Proof of Lemma 12}

\begin{proof}
Consider first the case of passive noise. Fix some $a \in \mathcal{A}$. Noting that $\langle \phi(a), \Delta_{\S} \rangle$ is the sum of $d$ terms bounded in $[-c,c]$, another application of Hoeffding's inequality and the union bound will show:
\[
    \mathbb{P}(\exists\, a \text{ s.t. } |\langle \phi(a), \Delta_{\S} \rangle| \geq \rho) \leq 2m\exp\left( -\frac{\rho^{2}} {2c^{2} d} \right).
\]
Therefore, with probability $1-\delta$, we have that $\Delta_{\S}$ is $\rho$-bounded for $\rho \leq c\sqrt{2d\ln(2m/\delta)}$. Noting that $L = \sqrt{d}$ exactly, the result follows by applying Theorem \ref{thm:noisy-decoding}.

Now let us consider the adversarial case in which $\|\Delta_{\S}\|_{1} \leq \omega s d$. We first observe that $|\langle \phi(a), \Delta_{\S} \rangle| \leq \|\phi(a)\|_{\infty}\|\Delta_{\S}\|_{1} \leq \omega s d$. Then, applying Theorem \ref{thm:noisy-decoding} we obtain:
\[
    \frac{\omega s d}{d} + s\mu < \frac{1}{2} \Rightarrow \omega < \frac{1}{2s} - \mu
\]
as claimed.
\end{proof}

\subsection*{Proof of Theorem 14}

\begin{proof}
Note first that $\|\phi(a) \otimes \psi(f)\|_{2} = \|\phi(a)\|_{2}$. Then, fixing $a,a'$ and $f$, by Hoeffding's inequality:
\[
    \mathbb{P}(|\langle \phi(a), \phi(a') \otimes \psi(f) \rangle| \geq \mu L^{2}) \leq 2\exp \left( -\frac{L^{4}\mu^{2}}{2\sigma^{2}||\phi(a')||_{2}^{2}} \right) \leq 2\exp \left(-\frac{k\mu^{2}L^{2}}{2\sigma^{2}}\right)
\]
where we have again defined $\kappa = (\min_{a} \|\phi(a)\|_{2}^{2})/(\max_{a'} \|\phi(a')\|_{2}^{2})$. The result follows by the union bound over all $< nm^{2}/2$ combinations of $a,a',f$.
\end{proof}

\subsection*{Proof of Theorem 17}

\begin{proof}
Expanding:
\[
    ||\phi(\bmm{x}) - \phi(\bmm{x}')||^{2}_{2} = ||\phi(\bmm{x})||_{2}^{2} + ||\phi(\bmm{x}')||_{2}^{2} - 2\langle \phi(\bmm{x}), \phi(\bmm{x}') \rangle
\]
Note first that $\|\phi(\bmm{x})\|_{2}^{2} = nd + \Delta$, where $\Delta$ is a mean-zero noise term due to cross-talk between the codewords. Neglecting minor errors from the ceiling function, the dot-product expands to:
\begin{align*}
    \langle \phi(\bmm{x}), \phi(\bmm{x}') \rangle &= \sum_{i=1}^{n} \langle \phi(x_{i}) \otimes \psi(f_{i}), \phi(x_{i}') \otimes \psi(f_{i}) \rangle + \sum_{i \ne j} \langle \phi(x_{i}) \otimes \psi(f_{i}), \phi(x_{j}') \otimes \psi(f_{j}) \rangle \\
    &= \sum_{i=1}^{n} \langle \phi(x_{i}), \phi(x_{i}') \rangle + \Delta'
    = \sum_{i=1}^{n} d(1-|a(x_{i}) - a(x_{i}')|) + \Delta' \\
    &= d(n - \|\bmm{x}-\bmm{x}'\|_{1}) + \Delta'
\end{align*}
where $a(x_{i})$ is taken to be the centroid corresponding to $x_{i}$ and $\Delta'$ is another noise term due to crosstalk. Putting both together and noting that $\Delta,\Delta' \leq n^{2}d\mu$ we have:
\[
    2d(\|\bmm{x}-\bmm{x}'\|_{1} - 2n^{2}\mu) \le \|\phi(\bmm{x}) - \phi(\bmm{x}')\|^{2}_{2} \le 2d(\|\bmm{x}-\bmm{x}'\|_{1} + 2n^{2}\mu)
\]
where the incoherence can be bounded as in Equation \ref{col:id-incoherence}.
\end{proof}